\documentclass{new_tlp}
\usepackage{mathptmx}
\usepackage{comment}
\long\def\comment#1{}

\usepackage[linesnumbered,ruled,slide,boxed,vlined]{algorithm2e}

\usepackage{amssymb,amsmath,verbatim,url,color}

\usepackage[english]{babel}

\newcommand{\clingo}{\textsc{clingo} }

\newcommand{\plasp}{\textsc{plasp} }

\newtheorem{theorem}{Theorem}[section]

\newtheorem{lemma}{Lemma}[section]
\newtheorem{corollary}[theorem]{Corollary}
\newtheorem{definition}{Definition}[section]

\title[Domain-Independent Cost-Optimal Planning in ASP ]{Domain-Independent Cost-Optimal Planning in ASP}

\author[David Spies et al.]
         {David Spies,  
          Jia-Huai You, Ryan Hayward  
         \\
University of Alberta, Edmonton, Canada
         }

\jdate{March 2003}
\pubyear{2003}
\pagerange{\pageref{firstpage}--\pageref{lastpage}}

\begin{document}

\maketitle

\begin{abstract}
 We investigate the problem of cost-optimal planning
in ASP.  Current ASP planners can be trivially extended to a cost-optimal one by adding weak constraints, but only for a given makespan (number of steps). It is desirable 
to have a planner that guarantees global optimality. 
In this paper, we present two approaches to addressing this problem. First,
we show how to engineer a cost-optimal planner 
 composed of two ASP programs running in parallel. 
Using lessons learned from this, we then develop an entirely new approach to cost-optimal planning, {\em stepless planning}, which is completely free of makespan. Experiments to compare the two approaches with the only known cost-optimal planner in SAT reveal  good  potentials for stepless planning in ASP. The paper is under consideration for acceptance in TPLP.
\end{abstract}



\section{Introduction}
{\em Answer Set Programming} (ASP) has been celebrated for its elegance and applicability to AI planning \cite{Lifschitz02}.
A planning problem in this context is, given a collection of actions, each with pre-conditions and effects, and fluents which are properties over states, determine whether there is a sequence of actions from an initial state to a final state \cite{DBLP:journals/ai/FikesN71}.
State-of-the-art ASP planners like \plasp have been developed \cite{gebser2011plasp,DimopoulosGLRS19}.

In cost-optimal planning, each action is associated with a cost, and
our objective is to find a plan which minimizes the sum cost of all
actions. Any ASP planner
can be trivially extended to a cost-optimal planner by adding weak
constraints, but only for a given makespan. Such a planner does not guarantee global optimality.  Eiter et al. \shortcite{eiter2003answer} present an approach to addressing this problem, but make the assumption that a
polynomial upper bound on makespan exists and is known in advance.


In the related field of SAT, 
some work has been done on cost-optimal Partial-Weighted-MaxSat
planning with regard to makespan (e.g., \cite{maratea2012planning}
and \cite{chen2008plan}). Again, readers should find this somewhat unsatisfactory.
After all, finding a plan that is ``cost-optimal with regard to makespan''
is just a way to sidestep the complication the real problem 
presents. A makespan is just an internal artifact
of the SAT approach to planning. A solution should not depend on
the way in which the planner happens to order the actions. Ideally,
we want an approach to planning which guarantees a globally optimal solution and
makes no mention of makespan.

We are aware of only one existing logic-based approach that tackles this much more
difficult problem for SAT planning, the paper by Robinson et al. \shortcite{robinson2010cost}. 
Inspired by this work, we pursue a separate investigation into globally cost-optimal
planning in ASP and develop a two-threaded planner,  one thread being a regular planner and the other an
{\em any-goal} planner. While the former computes successive decreasing upper bounds of the optimal cost by iteratively increasing the makespan, the latter computes
successive increasing lower bounds by planning in a modified environment where some amount of cheating is allowed. We achieve this by forcing the any-goal planner to ``make progress'' at each time-step.
An optimal plan is obtained when the two bounds meet in the middle, i.e., they are guaranteed to agree at some point.
However, unlike in Robinson's approach, through the use of ``make-progress" rules, we're able to develop a planner which guarantees to eventually find an optimal solution (or report no solution) even when the problem contains actions with zero cost.

Using insights gained from this approach, we then develop
a new approach to logic-based planning, {\em stepless planning}. The idea is to first engineer a planner which produces
``partially ordered plans" \-- actions arranged into a graph of dependencies where stable model semantics ensures that the graph is acyclic; then we show how to express the problem in $\Sigma_2^P$ of ``making progress"; and finally, we show a critical component of the planner, the {\em suffix layer},
which determines how many occurrences of each of the actions and fluents we will need to produce an optimal plan.

We report experiments on our cost-optimal planners on benchmarks of \cite{robinson2010cost} and compare with Robinson's SAT-based planner. We found that the stepless planner outperformed the other two planners in most domains and the two-threaded planner outperformed the SAT-based planner in most domains.

The paper is organized as follows. Section \ref{ASPPlan} describes ASP planning translated from SATPlan. Section \ref{no-solution} discusses the problem of no-solution detection and provides solutions.
Section \ref{no-solution-extension} extends no-solution detection  to optimal planning, with Section \ref{delete-free0} adding a {\em delete-free planner} as a suffix layer to improve the effectiveness.
Section \ref{2-threaded} gives the two-threaded planner and Section \ref{stepless} is about the stepless planner. Section \ref{experiments} reports experiments and Section \ref{related-work} is on related work and future directions.  


We assume that the reader is familiar with STRIPS planning. The ASP encodings in this paper are constructed to run on system \clingo and generally follow the ASP-Core-2 Standard \cite{ASP-standard}, except that we adopt two special features provided by \clingo\!:
(i) we will use ‘;’ to separate rule body atoms since the conventional comma sign ‘,’ is overloaded and has a different meaning in more complex rules, and (ii) the disjunctive head of a rule may be expressed conveniently by a conditional literal. 

Parts of this paper have been moved to Appendices, including proofs, encodings, and some technical explorations. For more information, the reader may also want to consult the thesis written by the first author of this paper \cite{spies2019}.



\section{Preliminaries: STRIPS Planning in ASP}
\label{ASPPlan}
We adopt a direct translation of 5 rules of
SATPlan \cite{kautz2004satplan04} into ASP and call the resulting planner ASPPlan.

\begin{verbatim}
rule 1.  holds(F,K) :- goal(F); finalStep(K).
rule 2.  happens(A,K-1) : add(A,F),validAct(A,K-1) :- holds(F,K); K > 0.
rule 3.  holds(F,K) :- pre(A,F); happens(A,K); validFluent(F,K).
rule 4.  :- mutexAct(A,B); happens(A,K); happens(B,K).
rule 5.  :- mutex(F,G); holds(F,K); holds(G,K).
\end{verbatim}
where $validAct(A,K)$ means that action $A$ can occur at time $K$ and $validFluent(F,K)$ means fluent $F$ can be true at time $K$.\footnote{Blum and Furst \shortcite{blum1997fast} give a handy way to identify for each
action and each fluent, what is the first layer at which this action/fluent
might occur by building the {\em planning graph}. Note that validAct/2 and validFluent/2 as well as predicates mutexAct/2 and mutex/2 are all extracted from the planning graph.}  Time-steps used in constructing a plan are also called {\em layers}.

Rule 1 says that goals hold at the final layer.
In rule 2, if a fluent holds at layer $K$, the disjunction of actions that have
that fluent as an effect hold at layer $K-1$. The next rule says that actions at each layer imply their preconditions. The last two rules are mutex constraints: in rule 4, actions with (directly) conflicting preconditions or effects are
mutually exclusive, and in rule 5, the fluents that are inferred to be mutually exclusive are encoded as constraints.


Following SATPlan, we add to our plan ``preserving'' actions for each fluent.
The goal is to simulate the frame axioms by using the existing
machinery for having an action add a fluent that gets used some steps later. 
These preserving actions can be specified as:

\begin{verbatim}
action(preserve(F)) :- fluent(F).
pre(preserve(F),F) :- fluent(F).
add(preserve(F),F) :- fluent(F).
\end{verbatim}
where each fluent $F$ has a corresponding {\em preserving action} denoted by term $preserve(F)$. Preserving actions can be easily distinguished from regular actions. 
Now that an action occurs at time $K$ indicates that
its add-effect $F$ will hold at time $K+1$.

Note that the reason why rule 5 of ASPPlan prevents fluents from being deleted before they're
used is a bit subtle.
In order for a fluent to hold, it must occur in conjunction with a preserving
action at each time-step it's held for. A preserving action has that fluent as a
precondition and so would be mutex with any action that has it as a delete
effect. This means that deleting actions cannot occur as long as that fluent
is held (by rule 4).\footnote{As a further note, when 
PDDL (planning domain definition language) without any extensions is defined, goals can only be positive and actions
can only have positive preconditions. There is a {\tt :negative-preconditions}
extension to PDDL, but we didn't use it.
Any problem which uses :negative-preconditions can be trivially
adapted to avoid using it by adding a fluent :not-F for every fluent :F and then
adding a corresponding add-effect wherever there's a delete-effect and vice
versa.}

Like SATPlan, we run this planner by solving at some initial makespan $K$, where $K$ is the first layer at which {\tt validFluent(F,K)} holds for all {\tt goal(F)},
and if it is UNSAT, we increment $finalStep$ by 1 until we find a plan.

This is a straightforward and unsurprising encoding in every respect,
but has a somewhat surprising consequence as compared to SATPlan. Because ASP models are
stable, for any fluent $F$, $holds(F,K)$ can
only be true if there exists some action which requires its truth
as per rule 3. Similarly for actions as per rule 2.
Furthermore, since rule 2 is disjunctive at every step, the set of actions which occurs is a minimal
set required to support the fluents at the subsequent step. This conforms
exactly to the approach to planning in \cite{blum1997fast}: First build
the planning graph, then start from the goal-state planning backwards,
at each step selecting a minimal set of actions necessary to add all
the preconditions for the current set of actions. That is, in this ASP translation, the neededness-analysis as carried out in \cite{robinson2008compact} is accomplished automatically during  grounding
or
during the search for stable models. 


\medskip
{\bf Encoding Reduction:}
Rule 4 in ASPPlan
can blow up in size when grounded because nearly any two actions
acting on the same
fluent can be considered directly conflicting.
For example, imagine a planning problem in which there is a crane
which we must use to load boxes onto freighters and there are many
boxes and many freighters available but only one crane. Then we will
have one such constraint for every two actions of the form,
$load(Crate,Freighter)$,
 for any crate and any freighter. As there is already a quadratic number of actions in the problem description size, 
the number of
mutex constraints over \emph{pairs} of actions is \emph{quartic} in the initial
problem description size.

We would like to avoid such an explosion by introducing new predicates
to keep the problem size down. We will only consider two actions to
be mutex if one deletes the other's precondition. But we will take
extra steps to ensure that no add-effect is later used if the same
fluent is also deleted at that step. Here is the revised encoding of
rule 4.

\begin{verbatim}
used_preserved(F,K) :- happens(A,K); pre(A,F); not del(A,F).
deleted_unused(F,K) :- happens(A,K); del(A,F); not pre(A,F).
:- {used_preserved(F,K); deleted_unused(F,K);
    happens(A,K) : pre(A,F), del(A,F)} > 1; valid_at(F,K).

deleted(F,K) :- happens(A,K); del(A,F).
:- holds(F,K); deleted(F,K-1).
\end{verbatim}

Effectively, we are splitting the ways in which we care that an action
$A$ can relate to a fluent $ F$ into three different cases: (i)
$A$ has $F$ as a precondition, but not a delete-effect; (ii)
$A$ has $F$ as a delete-effect, but not a precondition; and (iii)
$A$ has $F$ as both a precondition and a delete-effect.


By explicitly creating two new predicates for properties (i) and (ii),
we have packed this restriction into one big cardinality constraint.
Further, we must account for conflicting effects, so we define one
more predicate ({\tt deleted/2}) which encapsulates the union of all actions
from properties (ii) and (iii) (those that delete $F$) and assert that $F$ cannot
hold at this step if any of those actions occurred in the previous
one. 

Note that the formulation in \cite{blum1997fast} will not allow two actions $A$ and $B$
to ever happen at the same time-step $K$ if they have ``conflicting effects" (one
adds a fluent $F$ and the other deletes the same fluent). Our encoding
allows this, but only in cases where $holds(F,K)$ is false. 
Except for this technicality, the two are otherwise equivalent. A detailed justification is presented in Appendix B.

\section{Planning with No-Solution Detection}
\label{no-solution}

From a theoretical standpoint, let us consider why cost-optimal planning
is such a difficult problem. When the planner terminates with a plan of cost $c$, it is additionally asserting ``I have proved there does not exist a plan of cost $c-1$''.
But here we immediately have a problem because all the planners we have
written so far in ASP are not actually \textit{planners} in the sense of
\cite{blum1997fast}; they cannot identify when a problem has no solution.
If there is no solution
our ASP planners will simply
march on forever  searching for one until somebody kills the process.
So before we can create a cost-optimal planner, we must first create
a (normal) planner which can determine if a problem has no solution.

Planning is PSPACE-complete. 
What makes planning decidable is, of course, the finite
state space. Any plan which goes on too long will eventually visit some state twice, so we only need to search for a plan among those
that never revisit the same state. We can, of course, produce
a naive upper bound on the number of possible states by taking $n=2^{\left|fluents\right|}$
and then terminate the search after $n$ steps, but let us try to do 
better.
Indeed,
a main technical innovation of this work is the development of
ASP-based decision procedures for planning, which can potentially prove
the unsatisfiability of a planning problem much earlier than when taking
the theoretical number of states as upper bound for the planning horizon.


First, instead of requiring goals to hold at final step (rule 1 of Section \ref{ASPPlan}),
let us say
\begin{verbatim}
{holds(F,K)} :- fluent(F); finalStep(K).
\end{verbatim}

By using a choice rule here, the planner can now choose any goal it
wants and then plan towards that goal. This makes our instance always
satisfiable (just produce any valid sequence of actions, then take
the end-state and claim that was your goal).

Now here comes the ``we must make progress'' rule. We'll refer to this as the ``layered make-progress'' rule (which also includes its strengthening to be discussed in the next section)
to distinguish it from the ``stepless make-progress rule'' for stepless planning.

\begin{verbatim}
:- not holds(F,K) : not holds(F,J), fluent(F); step(J); step(K); J < K.
\end{verbatim}

In English: ``For any time-step pair $J$ and $K$ where $J<K$,
we cannot allow that every fluent $F$ which does not hold at $J$ also
does not hold at $K$.'
That is, any sequence containing time-steps $(J,K) ~(J < K)$ where the state of $K$ is a subset of that of $J$ would fail this rule.

Thus, the rule guarantees that there exists a makespan $n$ for which our
planning instance is UNSAT. This is because we are now enforcing that
the state must change at every time-step to take on some value which
it did not have in any previous time-step.  
But if there are only $m$
reachable states in our planning instance, then for all $n>m$ this
is clearly impossible.


Hence, we can build a complete planner by running two separate computations
in parallel. The first is our usual ASP planner which increases the
makespan until it finds an instance for which the solver finds
a plan. The second is our
``any-goal'' planner which increases the makespan until it finds
an instance for which the solver returns UNSAT, at which point we record
the previous makespan as $maxlength$.
Once that is done, if the
first instance manages to reach $maxlength$ and report UNSAT, we
can safely claim to have proved that no plan exists and terminate
the solver.

\subsection{Stronger Notions of Progress}\label{stronger-notions}

Unfortunately,
there exist problems that contain many independent
variables which may be separately manipulated to generate a large
easily-traversable state-space. For such problems, our solver above can
produce long plans which idly ``flip bits'' to avoid repeating themselves.

To make this scenario more concrete, imagine that we take any unsolvable
planning problem and adjoin to it a binary counter with one hundred
two-state switches. In addition to the actions from the original problem,
we also have two hundred actions which independently flip each of
the switches in the counter (either from $0$ to $1$ or from $1$
to $0$). Even though this counter has no impact on the problem itself,
it suffices to increase the length of the longest plan by a factor
of $2^{100}$ because for every state in the longest path, we can
flip through all possible arrangements of these switches before proceeding
to the next state. This easily puts the possibility of solving the
problem out of reach whenever there is no solution.

One way to deal with this is to somehow encode into our planner
the knowledge that the longest possible time it can take to iterate
over the possible states of two independent subproblems is the maximum
of the respective longest times rather than the product. One attempt at this
is given in Appendix A,
where we reduce the length of the
longest plan in the above example to $100$. 

We want to do better though and we can. We are able to formulate a stronger
definition of ``make progress'', which we conjecture perfectly
defeats the \emph{independent parts} problem in all its forms. 
First, a definition.
\begin{definition}
\label{A-partially-ordered-plan}A {\em partially-ordered plan}
is a transitive directed acyclic graph $G$ (equivalently, a partial
ordering) of ``action occurrences'' such that all topological sorts of
$G$ are valid sequential plans.
\end{definition}



Starting with any sequential plan $S$, we can generalize it to its
canonical partially-ordered plan as follows. If $a$ precedes $b$
in $S$, then we will say $a\prec b$ for actions $a$ and $b$ (adding
an edge from $a$ to $b$) iff any of the following holds:
\begin{enumerate}
\item $a$ adds some fluent which is used as a precondition for $b$
\item $b$ deletes some fluent which is used as a precondition for $a$ (and
$a\ne b$)
\item\label{rule-add-delete} $a$ adds a fluent which $b$ deletes
\item\label{rule-delete-add} $a$ deletes some fluent and $b$ adds the
same fluent 
\item $a$ and $b$ are different instances of the same action\footnote{Actually, this rule is not strictly necessary as any {\em minimal} plan which satisfies the other five rules {\em will also} satisfy it.  By keeping it, we don't have to worry about discussing {\em action occurrences} until later when we actually start digging into the stepless planner.}
\item There exists an action $c$ such that $a\prec c$ and $c\prec b$
\end{enumerate}

Note that these rules only apply to $a$-$b$ pairs for which
$a$ precedes $b$ in $S$ (otherwise, the statements \ref{rule-add-delete} and \ref{rule-delete-add} above would appear
to be contradictory).

If we add a source and sink $s$ and $t$ respectively to any partially-ordered
plan such that for all actions
$s\prec a\prec t$, we can consider
any $s$-$t$ cut $x$ as a generalized ``intermediate state'' for
this plan. To see this, take any ordering where the $s$-side actions
all precede the $t$-side actions and look at what fluents hold after
we have  taken only the $s$-side actions. Let us call this state $x$-state.

Here comes the strongest possible (domain-independent) definition
of ``make progress'' that we can think of.
The idea is that for any set of actions identified as being sandwiched in between two cuts, they must make progress by turning some fluent to true.

\begin{definition}\label{partial-progress}
A partially-ordered plan is {\em strongly minimal} iff, given any two $s$-$t$
cuts $x$ and
$y$, if there exists any $t$-side action in $x$ which is an $s$-side
action in $y$, then there must be some fluent which is true in $y$-state
but not in $x$-state. We similarly call a sequential plan strongly
minimal if its canonical partially-ordered plan is.
\end{definition}

An action is said to make progress if no two cuts exist on either side of
the action without this property (that some new fluent occurs between them).
In a strongly minimal plan, \emph{all} actions make progress.  

This beautifully handles the
one hundred 3-state switch scenario by forcing us, for each switch
$w$, to consider the generalized intermediate state where all $w$-flips happen before anything else. Thus if $w$ is flipped to the same state twice we can produce the cuts at each of those states demonstrating that this plan is not minimal.
The ASP encoding of this rule is complicated,
 and so will be
relegated to Section \ref{stepless} when we 
apply it to stepless planning.\footnote{The same technique \emph{can} with some effort be encoded for layered planners as well. Essentially, it requires quite a bit of boilerplate in order to talk about \emph{next} and \emph{previous} occurrences of each fluent and action.} This is the ``stepless make-progress rule''.
In our two-threaded planner introduced next, we have chosen to implement the layered make-progress rule.

\section{\label{no-solution-extension}Extending No-Solution
Detection to Cost-Optimality Detection}

Now, let us add weak constraints for action costs to ASPPlan (called {\em MinASPPlan}\footnote{Besides ASPPlan discussed here, it also includes a smart encoding of mutex constraints (cf. Chapter 6 of \cite{spies2019}).}) as well as to the any-goal solver
so that a plan with the least action cost can be identified by each solver for a given makespan.
We run MinASPPlan and the any-goal
solver in parallel.

Once the MinASPPlan solver produces a plan of cost $C$, we treat
$C$ as an upper-bound on the cost of an optimal plan. Then we
tell the any-goal solver to only search for non-repeating plans of
cost $<C$. When the any-goal solver terminates with UNSAT at some
time-step $n$ we claim that all non-repeating plans with cost
$<C$ have a makespan $<n$ which means we can stop the MinASPPlan solver after finishing with makespan $n-1$. Meanwhile, if the any-goal solver gives back a minimum-cost plan for layer $n$,\footnote{Note that make-progress rules are crucial \-- without them, the minimum cost will always be zero.}
then that cost \textit{is a
lower bound} on the optimal cost of any plan with makespan at least $n$. Thus, by increasing makespan, MinASPPlan computes successive lower (i.e., non-increasing) upper bounds of optimal cost and the any-goal solver computes successive increasinglower bounds of the optimal cost; when the two meet in the middle, an optimal plan is identified.

Let us summarize what we have so far.
We have two threads iteratively solve
successively larger instances.
We will name them I and II in deference to their similarity to
Robinson's \cite{robinson2010cost} Variant-I and Variant-II
encodings.
\begin{itemize}
\item I is MinASPPlan; 
 and
\item II is the any-goal solver augmented by weak constraints by action costs. It is similar to I except for two
major differences: (i) in place of the goal conditions, II is allowed to choose its own
goal, and (ii) II is given some notion of progress together with the constraint
that it \textit{must make progress} at every time-step.
\end{itemize}

How we determine when to stop depends on which solver (I or II)
lags behind. I's result costs are monotonically nonincreasing while II's result costs are monotonically nondecreasing. If I lags behind, then as soon
as I's lowest cost is $\le$ the II's cost for the layer it is
currently trying to solve, we can stop and report the solution at that layer as
optimal. If II lags behind, then
as soon as its cost for some layer is $\ge$ the best known I-cost
so far (at \emph{any} layer), we can stop.

The asymmetry in deciding when to stop happens because of the types
of bounds I and II produce. I will never produce a cost $D<C$
if $C$ is the optimal cost, but II will continually increase its
lower bound eventually marching straight past $C$ and on to infinity
(the point at which it returns no solution). This is why we must take
into account the layer at which each lower bound was produced when
determining if we are done, but we do not care what layer the upper
bound was produced at.

\section{Delete-Free Planning}
\label{delete-free0}
The lower bounds produced by the Variant-II solver can be improved by adding a {\em suffix layer}, which is a {\em delete-free planner}.
Delete-free planning (DFP) is a special case of planning which happens to
be in NP. These
are the planning problems without delete-effects. 
Surprisingly, DFP can be modeled as a graph problem. 

Given a directed bipartite graph $G=\left(X,Y,E\right)$ with weights
on $X$ and a goal set $Y_{F}\subseteq Y$, find a minimum \textit{acyclic}
subgraph $G^{*}=\left(X^{*},Y^{*},E^{*}\right)$ such that
\begin{enumerate}
\item $Y_{F}\subseteq Y^{*}$
\item If $x\in X^{*}$ and $\left(y,x\right)\in E$, then $y\in Y^{*}$
and $\left(y,x\right)\in E^{*}$
\item For all $y\in Y^{*}$, $E^{*}$ contains at least one edge $\left(x,y\right)$
(and $x\in X^{*}$).
\end{enumerate}

{\bf Connection to DFP:}  $X$ is the set of actions, and $Y$
is the set of fluents, the $\left(x,y\right)$ edges are add-effects
and the $\left(y,x\right)$ edges are preconditions. $Y_{F}$ is the
goal set and the initial set has been removed (together with all corresponding
preconditions) from the graph. Rule 1 means the goal fluents must be
true. Rule 2 means an action implies its preconditions. Rule 3 means
every fluent must have a causing action. The graph must be acyclic to
ensure the actions can occur in some order. This is possible because
there is no incentive to ever take an action or cause a fluent more than
once. As soon as any fluent is true, it is permanently true.

We now can encode DFP in ASP as solving the above graph problem {\em independent of makespan}. 
Its encoding in ASP can be found in Appendix C.

The key takeaway is that the encoding is an efficient ``one-shot" encoding in ASP.
Rather than structuring the problem into layers and then iteratively increasing
the makespan until a solution is found, we eschew layers entirely and encode
the problem as a single ASP instance. This is similar to how (unlike with SAT encodings)
Hamiltonian Path
can be encoded one-shot in ASP \cite{spies2019} without needing
numeric fluents or layers or quadratic space when grounded.
The problem of ensuring that an acyclic structure exists is solved by 
the stable-model semantics.

This may be considered the most novel contribution of this paper. Besides for delete-free planning
(and the suffix layer in the next section),
the same trick will also be used later for "Stepless Planning" where we arrange actions
and fluents into a graph and rely on stable-model semantics to ensure that the
graph is acyclic. Without the ability to do this, other encodings of planning
problems are forced to rely on either using numeric fluents (not compatible with
SAT-based techniques) or structuring the problem into layers (which multiplies
the grounded size of the problem by the number of layers needed).

\section{A Two-Threaded Cost-Optimal Planner with Suffix Layer}\label{2-threaded}
As with $A^*$-search, we can generate successively better lower bounds by
planning normally from the starting state $Q$ to some intermediate
state $S$ chosen by the planner and then finding the minimum-cost solution to the delete
relaxation for the planning problem from $S$ to the goal state.\footnote{
  More precisely, we encode in ASP the problem of finding the minimum \emph{total} cost
  across all possible subgoal states
  of $cost(normal plan) + cost(suffix plan)$ (given that the normal plan respects
  whichever progress rule we choose to employ)
}
This suggests a natural way to modify our Variant-II encoding in Section
\ref{no-solution-extension} to find better lower bounds.
We append a ``suffix layer'' at the end, which must generate a plan
in the delete relaxation of the problem from the chosen any-goal state
to the actual goal state. The costs for any actions taken in the suffix
layer must be added to the total cost of our plan. Indeed, in many
cases this produces a remarkable lower bound.

We now give a complete description of our two-threaded planner. We have two ASP programs running in parallel. One is the
Variant-I standard ASPPlan solver with weak constraints for action costs. The other is the
Variant-II solver with a progress rule and appended a suffix layer.

\begin{itemize}
  \item Both solvers independently run successively on makespan $0$, $1$, $2$
    etc. until we kill them.
  \item Each time the Variant-I solver begins solving a new makespan, we update
    the current makespan being solved for.
  \item When the Variant-I solver finds a plan, we record the plan and its cost
    if this is the lowest-cost plan found so far.
  \item When the Variant-II solver finds an optimal plan for some makespan
    using the suffix layer,
    we record the optimal cost as a lower bound for that makespan
    (as well as all larger makespans).
  \item If the Variant-II solver ever finds an optimal plan which \emph{doesn't}
    use the suffix layer, then that plan is globally optimal. We can return it as
    a solution and entirely ignore the Variant-I solver (this only happened
    twice in all of our experiments and seems to be fairly unlikely).
  \item If the Variant-II solver obtains
    UNSAT for a layer, we can stop running it and record the cost of that and
    all future layers as $\infty$.
  \item Any time the best-cost plan found so far (by the Variant-I solver) is
    no greater than the Variant-II lower bound for the currently-solving layer or
    \emph{any} earlier layer, we can stop both solvers and report that plan as
    an optimal solution.
  \item Any time the Variant-I solver is solving for a makespan whose Variant-II
    lower bound is $\infty$, we stop the solver and return the best-cost
    plan found so far or ``no solution'' if no plan has been found.
\end{itemize}

The correctness of our two-threaded planner depends on the correctness of two component solvers. While this is straightforward for Variant-I solver, we have the following claim for Variant-II solver.

\begin{theorem} Let $Q$ be a planning problem. Assume that we are using the layered `make-progress' rule.
\begin{itemize}
\item (Variant-II Soundness)
If $Q$ has an optimal solution $P$ with cost $C$ and makespan $n$, then for any $k \leq n$,
at makespan $k$ the Variant-II solver will find a relaxed plan with cost $\leq C$.
\item (Variant-II Completeness)
If $Q$ has no solution, then the Variant-II solver will eventually produce an UNSAT instance.
\end{itemize}
\end{theorem}

Intuitively, the completeness is due to non-repetition of states enforced by the `make-progress' rule, and the soundness follows from the fact that if there exists a plan, then it can be ``reduced” to a plan
which satisfies the `make progress' rule. Furthermore, it can be truncated at any lower makespan
to a partial plan which satisfies the `make progress' rule.
A more detailed argument (and all proofs of the claims of this paper) can be found in Appendix B.

\section{Planning without Layers: Stepless Planning}
\label{stepless}


Besides the delete-free planner above (which is only useful for delete-free
planning problems), all the planners so far in this paper
(and indeed, all SAT/ASP planners
that we have encountered) have used layers to order the actions and fluents
that occur within a planning problem. But let us consider the notion of \textit{partially ordered plan} from Definition \ref{A-partially-ordered-plan},
where no layers are specified. 
Any topological sort of this
graph corresponds to a valid plan. Perhaps we could avoid layers entirely and embed action-dependencies directly.
The idea here is that, just as with delete-free planning, we can create
a plan by specifying only which actions and which fluents hold, and
we will rely on stable model semantics to ensure that the resulting solution graph
is acyclic.

As was mentioned in Section \ref{delete-free0} this is the most novel
contribution of this paper. We use stable-model semantics, rather than layers,
to produce an acyclic plan.

There is a key difference between delete-free planning and full stepless planning though,
which accounts for the distinction in computational complexity.
In the case of delete-free planning, no fluent holds more than once and no action
occurs more than once. In stepless planning, it is possible for an action to
occur multiple times. As such, we will have to have separate atoms in our encoding
representing each \emph{occurrence} of an action.
But prior to solving, we don't know how many occurrences of each action or fluent
will be needed.

Here, we will first present a solver that \emph{assumes}
it has enough occurrences and then 
we will come
back to the issue of figuring out how many of each are needed in order to
produce an optimal plan.
The stepless planner is \emph{significantly} more complicated than
anything else done in this paper so we put more care into explaining what each
line of ASP code does, but we will have to do it in an appendix (Appendix D). Here we provide an outline of the planner.  Additionally, since no planner
like this has ever been built before, we will take more care to try and
bridge the gap between the standard approach to planning and the approach being
presented here. 

\subsection{Stepless Planner Encoding}
To avoid
an $O\left(\left|\mbox{actions}\right|^{2}\right)$-size encoding,
we don't directly encode dependencies between actions. Instead we use the fluents as
intermediate nodes in the solution graph.

An \emph{occurrence} of a fluent $F$ will be encoded as an object in an atom,  $fluentOcc(F,M)$,
where $M$ is a sequentially-ordered index. $M=0$ is reserved for the initial
fluents. All others start at $M=1$ (when caused by some action).
Similarly, the object $
actOcc(A,N)$
indicates an occurrence of action $A$.\protect\footnote{We index action and fluent
occurrences with numbers $N$ and $M$ and have symmetry-breaking rules ensuring
that the occurrences happen in numerical order for a given action or fluent,
but it is important to understand
that these numbers are \emph{not} layers. There's no global \emph{step} of any
kind to which they correspond. A fluent occurrence can be used as a precondition
for an appropriate action occurrence regardless of what their indices
are or how they relate to each other.
The same goes for an action \emph{causing} a fluent. The indices are simply to
be able to distinguish between multiple occurrences of the same object; they
have no global significance or relation to any other object.}
In stepless planning, there are no preserving actions since there are no layers
to preserve things across, and we don't utilize mutex relationships between objects. 
Whereas in our previous encodings the causes and destroyers of each
fluent were implicit, here we must explicitly give which occurrence
of which action $AO$ causes which fluent occurrence $F\!O$ to hold ($causes(AO,F\!O)$) and which occurrence of
which fluent $F\!O$ is used as a precondition for which action occurrence $AO$ ($permits(F\!O,AO)$).
Additionally, we need an atom for each deleted fluent occurrence $F\!O$ which
action occurrence $AO$ has as a precondition and deletes ($deletes(AO,F\!O)$) and
also one in the rare case that an action has a fluent as
a delete-effect, but not a precondition, for which occurrence of the fluent $F$ the
action occurrence $AO$ follows ($follows(AO,F\!O)$). (Refer to Appendix D
under the subtitle {\em Problem Description}.)

From this we can structure the graph and assert that it is acyclic. For each
action occurrence we have an ``event''; additionally there is an event for the
start and end of each
occurrence of each fluent. There is also an event ``goal'' which corresponds to
the goal state being reached.

Events are grouped into \emph{vertices} in our graph
each of which contains at most one action occurrence.
When an action occurrence \emph{causes} a fluent, the action and
the start of that fluent belong to the same vertex. Similarly when it
\emph{deletes} a fluent, the action and the end of the fluent belong
to the same vertex.
To encode this we use the predicate $inVertex/2$ which indicates that its first
argument belongs to the vertex named by the second argument. (Refer to Appendix D 
under
the subtitle {\em Plan Event Graph}.\protect\footnote{We have removed the goal event from the encoding in the appendix since it is incompatible with the suffix layer. To see the original encoding refer to \cite{spies2019}.})

\subsection{Making Stepless Progress with a Suffix Layer}
\label{progress-suffix}
Now we need a way to assert that the action occurrences of a given stepless plan ``make
progress''. With no layers to make assertions about, the only notion of
progress we are left with is the definition (Def. \ref{partial-progress})  of a plan which is {\em strongly minimal}.  This definition logically takes the form of ``there does not
exist a pair of $(s,t)$-cuts such that ...." This means that \emph{given} a
particular plan, determining whether it is strongly minimal is likely
co-NP-complete (membership is straightforward but the hardness is an open conjecture), and then the problem of \emph{determining the existence of}
such a plan for a given
collection of atoms and fluents could possibly be $\Sigma_2^P$-complete.
Luckily,
ASP gives us a way to encode problems in $\Sigma_2^P$ through the use of \emph{disjunctive rules} \cite{Baral:aspbook}.
The code for this can be found in Appendix D 
 under the subtitle {\em Strong Minimality}.

If there aren't enough occurrences of a fluent or action, we
can tack on a suffix layer in the same way we did with the stepped-cost-optimal
planner. In the code in Appendix D \ref{appendix-stepless} 
under the subtitle  {\em Suffix Layer}, we replace all uses of $goal$ with a $subgoal$ which is the
entry-point into the suffix layer. 
The coding is similar to the suffix layer used in the two-threaded planner,
but there are a few key differences. First, if the suffix layer is
used at all, we use an atom {\em useSuffix} to indicate that this is true. There is a cost of 1 at level -1
for
{\em useSuffix} so among plans of equal cost, the solver will prefer one which
\emph{doesn't} use the suffix to one which does. If an optimal solution
doesn't use the suffix, then it must be globally optimal with respect to cost. (Refer to Appendix D 
under subtitle {\em Suffix Layer}.)


Finally, we will add rules to enforce the use of action and fluent occurrences from our bag so
that the planner resorts to the suffix layer only when it ``runs out'' of
something.
With this we know how to expand our bag of occurrences. Each time we get
back a plan making use of the suffix layer, look at all the fluents or
actions which were saturated by that plan and add another occurrence of each
one. 
(Refer to Appendix D 
under subtitle {\em Saturated}.)

This, coupled with our definition of making progress, is
what guarantees that it will eventually find a plan or determine that none exists.
The suffix layer is only
used because the planner ran out of something it needed and needs to request
more of that item from the controlling program (in particular, not as a way to save on plan
cost).

Appendix E 
provides a detailed example of running the stepless planner, and Appendix B
gives a proof of the following theorem.

\begin{theorem}
\begin{itemize}
\item
(Stepless Soundness)
All plan costs produced by the stepless planner are lower bounds on the cost of the true optimal plan.
\item (Stepless Completeness) The stepless planner will eventually find the solution if it exists or produce
an UNSAT instance if it doesn't.
\end{itemize}
\end{theorem}

\section{Experiments}
\label{experiments}

We ran our cost-optimal two-threaded solver and stepless solver on most of the same
instances as Robinson \cite{robinson2010cost} and
here report results\footnote{excluding
$satellites$ since our planner doesn't support the $:equality$ extension to PDDL
and $miconic$ since we couldn't find the problem files for it. Robinson was kind
enough to send us the instances from his constructed $ftb$ domain so we can
report performance on that as well}.

Experiments were run on a cluster of $c3.large$ Amazon EC2 instances each with
two Intel Xeon 2.8 GHz CPU cores and 3.75 GB of memory. We used
GNU Parallel \cite{tange_ole_2018_1146014} to distribute the work of running
multiple instances.

For comparison, we include Robinson's reported results scaled down by a factor of
$\frac{2.6}{2.8}$ to account for the difference in processor speeds.

For each domain, we report the largest instance solved by each of the two-threaded planner,
the stepless planner, and Robinson's planner where largest is measured by the amount of time
it took that planner to solve the instance. Where it differs, we also report the
largest-indexed instance solved by each of the two-thread and stepless
planners.

Every plan produced by either planner was validated by the Strathclyde Planning
Group plan verifier VAL \cite{howey2004val}.

The column $C^*$ is the optimal cost found for each instance. In all cases the optimal cost
for the two-thread planner agrees with the
optimal cost reported by Robinson \cite{robinson2010cost} where applicable
(Robinson compares his results against a \emph{non}-SAT-based planner and our
optimal costs agrees with that as well).

\begin{table}[t]
\footnotesize
\caption{Experiments with two-threaded planner and stepless planner} 
\centering 
\begin{tabular}{| c ||c c c c c c c c c c |}
\cline{1-11} 
Problem & $C^*$ & $n$ & $n_*$ & $t_\pi$ & $t_*$  & $t_{2\mbox{-}threaded}$ & $n_s$ & $t_{stepless}$ & $l_s$ & $t_{Robinson}$  \\
\cline{1-11} 
block-12 & 20 & 20 & 17 & 0.5 & 1203.4 & 1203.9 & - &- &- &?  \\
block-15 & 16 & 16 & 12 & 0.4 & 113.4 & 113.8 & 7 & 89.4 & 33.7  &? \\
block-18 & 26 & 26 & 16 & 0.9 & 256.8 & 257.7 & - & - & -  & 3.2 \\
block-23 & 30 & - & - & - & - & - & - & - & -  & 29.8 \\
block-25 & 34 & - & - & - & - & - & - & - & -  & 27.4 \\
depots-2 & 15 & - & 12 & - & 771.5 & 771.5 & 2 & 9.7 & 4.2 & - \\
depots-13 & 25 & - & - & - & - & - & 3 & 475.9 & 137.2 & -\\
driverlog-2 & 19 & - & - & - & - & - & 20 & 215.5 & 44.4 & -\\
driverlog-3 & 12 & 7 & 3 & 0.1 & 0.9 & 1.0 & 1 & 0.4 & 0.4 & 450.2\\
driverlog-11 & 19 & - & - & - & - & - & 1 & 13.5 & 13.5 & -\\
elevators-2 & 26 & 3 & 0 & 0.4 & 1.7 & 2.1 & 1 & 2.7 & 2.7 & 13.0\\
freecell-3 & 18 & - & - & - & - & - & 2 & 420.5 & 344.0 & -\\
ftb-30 & 1001 & 25 & 0 & 1.8 & 0.3 & 2.1 & 1 & 5.5 & 5.5 & 1.8\\
ftb-38 & 601 & 33 & 0 & 2.7 & 0.2 & 2.9 & 1 & 3.2 & 3.2 & 1.5\\
ftb-39 & 801 & 33 & 0 & 3.9 & 0.3 & 4.2 & 1 & 5.6 & 5.6 & 2.2\\
ftb-40 & 1001 & 33 & 0 & 3.9 & 0.4 & 4.3 & 1 & 8.2 & 8.2 & ?\\
gripper-1 & 11 & 7 & 4 & 0.1 & 0.4 & 0.5 & 2 & 0.4 & 0.2 & 14.6\\
gripper-2 & 17 & 11 & 8 & 0.6 & 312.4 & 313.0 & 7 & 23.5 & 9.7 & -\\
pegsol-9 & 5 & 15 & 11 & 3.9 & 35.9 & 39.8 & 5 & 131.5 & 46.6 & 386.8\\
pegsol-16 & 8 & 21 & 17 & 48.3 & 1029.0 & 1509.3 & 10 & 910.2 & 280.8 & -\\
pegsol-18 & 7 & - & - & - & - & - & 7 & 1548.0 & 537.1 & -\\
rovers-3 & 11 & 7 & 4 & 0.1 & 0.2 & 0.3 & 1 & 0.1 & 0.1 & 49.4\\
rovers-4 & 8 & 4 & 0 & 0.0 & 0.0 & 0.0 & 1 & 0.1 & 0.1 & ?\\
rovers-6 & 36 & - & - & - & - & - & 48 & 1354.3 & 391.0 & -\\
rovers-9 & 31 & - & - & - & - & - & 53 & 1040.6 & 101.4 & -\\
rovers-14 & 28 & - & - & - & - & - & 72 & 900.7 & 55.9 & -\\
storage-7 & 14 & 14 & 11 & 0.6 & 42.9 & 43.5 & 10 & 89.2 & 42.4 & 1.1\\
storage-8 & 13 & - & - & - & - & - & 15 & 799.1 & 239.5 & -\\
storage-9 & 11 & - & - & - & - & - & 9 & 181.0 & 46.0 & -\\
storage-13 & 18 & - & - & - & - & - & - & - & - & 244.0\\
TPP-5 & 19 & 7 & 2 & 0.1 & 0.2 & 0.3 & 2 & 0.5 & 0.3 & -\\
TPP-7 & 34 & - & - & - & - & - & 13 & 189.6 & 32.4 & -\\
transport-1 & 54 & 5 & 0 & 0.1 & 0.1 & 0.2 & 2 & 0.5 & 0.3 & 0.2\\
transport-2 & 131 & 12 & 4 & 74.3 & 55.1 & 129.4 & 2 & 111.6 & 106.3 & -\\
transport-11 & 456 & 9 & 3 & 0.3 & 1.6 & 1.9 & 2 & 163.4 & 151.4 & -\\
transport-21 & 478 & 7 & 1 & 0.2 & 0.6 & 0.8 & 2 & 5.2 & 3.5 & -\\
zenotravel-4 & 8 & 7 & 3 & 0.5 & 2.8 & 3.3 & 3 & 14.1 & 6.9 & 783.4\\
zenotravel-6 & 11 & 7 & 0 & 7.2 & 6.5 & 13.7 & 1 & 2.1 & 2.1 & -\\
zenotravel-10 & 22 & - & - & - & - & - & 1 & 1387.1 & 1387.1 & -\\
\cline{1-11}
\end{tabular}
\label{table:nonlin}
\end{table}

The column $n$ is the lowest makespan at
which the problem has a $C^{*}$ plan (according to our Variant-I solver).
Our value for the makespan $n$ agreed with all of Robinson's
reported results except for Rovers-3 where we found we only needed a makespan of
7 to produce the optimal plan while Robinson reported a required makespan of 8.
We suspect this is because the definition
of mutex 
of \cite{blum1997fast} is overly restrictive for
actions (cf. the footnote on page 3).


$n_{*}$ is the makespan
at which our Variant-II suffix solver proves $C^{*}$ is optimal. Interestingly,
for many of instances this value was $0$ which indicates that the
optimal plan in the delete-free reduction of the problem has the same cost as
the true optimal plan.

$t_{\pi}$ is the time required to find the plan (by our Variant-I  solver);
$t_{*}$ is the time required to prove optimality (by our Variant-II suffix solver); and
$t_{2\mbox{-}threaded}$ is the sum of these two numbers (can be thought of as ``total solve time'' although the algorithm necessitates that they run in parallel, so
the actual wall-clock time required to run them was the maximum, not the sum,
but with two CPU cores rather than one).
All reported times are measured in seconds.

$n_s$ is the number of times the stepless solver was run for this instance (each
time adding more items to its bag of fluents and actions based on what was
saturated in the previous rounds). On the last of these runs it produced an
optimal solution which doesn't use the suffix layer and hence is globally
cost-optimal. $t_s$ is the total time running the stepless solver across all
runs.

One important distinction between the experiments run with the two-threaded solver
and those run with the stepless solver is that the two-threaded solver took
advantage of ``iterative'' solving.
That is, \clingo provides an API for interacting with it programatically. Rather
than restarting from scratch each time there is a new instance to be solved, we
can after observing the solution at makespan $k$, make some adjustments to the
instance so that it now represents the program for makespan $k+1$ and then ask
\clingo to {\em continue} solving from this point while maintaining any learnt
clauses which are still relevant.

This is incredibly powerful and resulted in a major speedup in the two-threaded
solver.

For the stepless solver this was not possible since it relies heavily on
full-program-spanning loop constraints to give correct results, but \clingo doesn't support having loop
constraints cross multiple iterative stages. Thus, every time new fluents
or actions are added to the stepless solver's bag, it starts solving
from scratch. In the future we hope \clingo can support this, but they have no
plans to do so at this time.

$l_s$ is the total time required for the last iteration
of the stepless solver
to run. This one run by itself is sufficient to both find the globally optimal
solution \emph{and} prove its optimality. However we know of no more efficient way to
find the right bag of actions and fluents in order to guarantee the optimal
solution won't use the suffix layer.
This number is still interesting in that it provides a lower bound on the time
it would take to solve the instance if \clingo supported loop constraints
crossing program section boundaries (so that we could add more occurrences and
continue solving rather than having to restart). It gives us some idea of what
savings such a modification to \clingo might provide.

$t_{Robinson}$ is the total time reported by \cite{robinson2010cost} to find the optimal
solution scaled by
a factor of $\frac{2.6}{2.8}$. A question mark $?$ in this column indicates the time is unknown
since it's not reported in \cite{robinson2010cost}. If the solver for which this
row is maximal successfully solved the largest instance reported by Robinson in this domain and
found this instance to be larger, we fill with a dash mark $-$ rather than $?$ in this column 
(our best guess as to whether Robinson's planner solved it).  
A  $-$ in any other column indicates the relevant planner failed to solve the
instance in less than 1671.4 seconds (30 minutes scaled down by $\frac{2.6}{2.8}$).  
In the case of depots-2, the Variant-II suffix solver reached layer 12 before
the Variant-I MinASPPlan solver and so it found an optimal no-suffix solution
by itself.

All instances were solved with \clingo version 5.2.3. The controller logic for
both the two-threaded
solver (handling of incremental solving, coordinating the two solvers, and
figuring out when to terminate the search) and the stepless solver (figuring out
which occurrences to include and topologically sorting the output) was
written in Haskell using the clingo-Haskell bindings written by tsahyt (GitHub
alias) to communicate with \clingo.

We used the default configuration and options for \clingo except that the
stepless planner used the --opt-usc option which finds optimal solutions by
expanding an unsatisfiable core (Definition 2 in \cite{alviano2015optimum}).

All planners presented here are available
on GitHub at
\url{https://github.com/davidspies/aspplan2}.
Feel free to contact the repo owner (the first author of this paper) for any help
with reproducing these results.

In all domains except for $blocks$, $ftb$, and $storage$, our two-threaded solver outperformed
Robinson's SAT-based solver and our stepless solver outperformed both
(in terms of number of instances solved).
In the case of $storage$, the stepless solver and Robinson's solver each solved
an instance which the other failed to solve which seems to point to the possibility
that the stepless solver encounters different difficulties from a more traditional
approach.
One more piece of evidence favoring this conclusion is that the toy example
bridge-crossing problem from \ref{example-bridge-crossing}
required a full
30 seconds to solve (whereas the 2-threaded solver solves it in 2
seconds) and in general we found that on small/toy problems
the stepless solver's performance is abysmal compared with other approaches
we tried but 
scales better with larger instances.

Prior to running the full suite of experiments, the above observation gave us the mistaken impression
that the stepless solver was interesting as a theoretical oddity, but fails to
produce decent results in practice, since for every example we ran it on while tuning
it, it seemed to run slower than the two-threaded solver. It was a pleasant
surprise to discover when officially running the experiments that in fact the
inverse was true.

\section{Related Work and Final Remarks}
\label{related-work}

Our 2-threaded solver algorithm is inspired by  the approach of
 \cite{robinson2010cost}.
The better performance of our planner,
 besides solver technologies,
 seems partly due to the grounding size and search space pruning under stable model semantics as commented in Section \ref{ASPPlan}, and partly due to less clustered encoding in ASP than in SAT, plus the smart encoding of mutex constraints. But note that
their approach works only if we assume the problem is solvable (which he does) and all actions have positive (nonzero) costs (which he also does).

Eiter et al. \shortcite{eiter2003answer} propose an approach to
finding globally cost-optimal plans in ASP with action costs, but confine their discussion to the planning problems which look for {\em polynomial length plans}. They make the assumption
that the planning domain has some polynomial
upper bound on plan lengths which is known in advance.  In contrast, we do not make such an assumption.  It is interesting to note that through some key technical innovations, we are able to show that the current ASP techniques are capable of encoding cost-optimal planners without this assumption. 

Partially-ordered plans have been explored elsewhere. One  is the CPT planner \cite{VidalG06}
for optimal temporal planning using Constraint Programming (CP), where actions have {\em durations} and makespan refers to total duration which corresponds to the cost of a plan in our setting. Optimality here means minimal duration.
CPT consists of pre-processing that induces lower bounds to be used in starting makespan and in formulation of constraints, a branching scheme, and a CP-based branch and bound search. The branching scheme is specifically designed for temporal planning aiming for  increased reasoning efficiency. The current makespan $B$ increases by 1 if no plan is found. Thus, the first plan found is guaranteed to be duration minimal.
CPT dynamically generates {\em action tokens} from {\em action types}
 achieving the similar goal of the suffix layer of our stepless planner.
Thus, a main difference between CPT and our stepless planner is no-solution detection in our case and  the lack of it in CPT, i.e., the current CPT does not terminate on its own for UNSAT instances.

Delete-free planning has been investigated as a stand-alone topic, including a CP solution \cite{BartakDGBT12}. To the best of our knowledge, our modeling of delete-free planning as a graph problem is original, and it leads to a five-line ASP program which does everything (cf. Appendix C).

The {\em Madagascar Planner} is a family of efficient implementations of the SAT based techniques for planning. The main idea is, instead of using
the standard decision
heuristics such as VSIDS,  planning-specific variable selection heuristics are applied \cite{Rintanen12}. One would expect that the same idea can work for ASP-based planning, and in this case, our 2-threaded cost-optimal planner can benefit from it directly.

Cost-optimal planners can also be built on the $SAS^+$ platform. An $SAS^+$ planner based on greedy selection of a subset of heuristics for guiding $A^*$ search \cite{LelisFABZH16} has made to the top tier in IPC-2018. $SAS^+$ planning can be encoded in SAT and ASP as well, but the most critical component, the selection algorithm, needs to be implemented by an external program.

Though the goal of this paper is limited to ASP-based cost-optimal planning, there is always a question of whether such a planner is competitive in efficiency (in terms of solving time) with state-of-the-art planners, e.g., the top planners from IPC.
Further investigation and experimentation are needed to address such questions. 
A major advantage of ASP-based planners is the succinctness and elegance of the encoding. An expressive KR language like ASP provides some unique advantages, e.g., determining the existence of a plan that satisfies the stepless make-progress rule is likely $\Sigma_2^P$-hard, which can be a challenge for other KR languages.

 Stepless planning is a brand new approach to logic-based planning
and
brings with it a lot of unknowns and potentials for future directions. One issue is that the lack of any notion of simultaneity makes certain standard optimizations difficult, such as incorporating mutex constraints and supporting {\em conditional-effects} (an extension to STRIPS planning).
The latter extension has been realized in SAT-based planning \cite{Rintanen11}, but incorporating it to stepless planning appears to be non-trivial.
Our stepless planner is a nontrivial application that requires supportedness
cycles to extend across different program sections, it would be nice if
 \clingo supported iterative solving with this.

More recently, no-solution detection for planners has become an interesting topic, along with the competition called Unsolvability IPC, which aims to test classical automated planners to detect when a planning task has no solution. Our no-solution techniques presented in this paper may be relevant. This is one interesting future direction. 

Finally, {\em property directed reachability} (PDR), a promising method for deciding reachability
in symbolically represented transition systems, which was originally
conceived as a model checking algorithm for hardware
circuits,
has recently been related to planning 
\cite{DBLP:journals/jair/Suda14}. The relationship with our stepless planner deserves a further study; in particular, an interesting question is whether and how PDR-based planners can be strengthened to become cost-optimal. 


\bibliographystyle{acmtrans}
\bibliography{david}

\begin{thebibliography}{}

\bibitem[\protect\citeauthoryear{Alviano, Dodaro, Marques-Silva, and
  Ricca}{Alviano et~al.}{2015}]{alviano2015optimum}
{\sc Alviano, M.}, {\sc Dodaro, C.}, {\sc Marques-Silva, J.}, {\sc and} {\sc
  Ricca, F.} 2015.
\newblock Optimum stable model search: algorithms and implementation.
\newblock {\em Journal of Logic and Computation\/}.

\bibitem[\protect\citeauthoryear{Baral}{Baral}{2003}]{Baral:aspbook}
{\sc Baral, C.} 2003.
\newblock {\em Knowledge Representation, Reasoning and Declarative Problem
  Solving}.
\newblock Cambridge University Press, New York, NY.

\bibitem[\protect\citeauthoryear{Bart{\'{a}}k, Dvorak, Gemrot, Brom, and
  Toropila}{Bart{\'{a}}k et~al.}{2012}]{BartakDGBT12}
{\sc Bart{\'{a}}k, R.}, {\sc Dvorak, F.}, {\sc Gemrot, J.}, {\sc Brom, C.},
  {\sc and} {\sc Toropila, D.} 2012.
\newblock When planning should be easy: On solving cumulative planning
  problems.
\newblock In {\em Proc. 25th International Florida Artificial Intelligence
  Conference}, Florida, USA.

\bibitem[\protect\citeauthoryear{Blum and Furst}{Blum and
  Furst}{1997}]{blum1997fast}
{\sc Blum, A.~L.} {\sc and} {\sc Furst, M.~L.} 1997.
\newblock Fast planning through planning graph analysis.
\newblock {\em Artificial Intelligence\/}~{\em 90,\/}~1, 281--300.

\bibitem[\protect\citeauthoryear{Calimeri, Faber, Gebser, Ianni, Kaminski,
  Krennwallner, Leone, Ricca, and Schaub}{Calimeri et~al.}{2015}]{ASP-standard}
{\sc Calimeri, F.}, {\sc Faber, W.}, {\sc Gebser, M.}, {\sc Ianni, G.}, {\sc
  Kaminski, R.}, {\sc Krennwallner, T.}, {\sc Leone, N.}, {\sc Ricca, F.}, {\sc
  and} {\sc Schaub, T.} 2015.
\newblock {ASP-Core-2} input language format.
\newblock \url{https://www.mat.unical.it/aspcomp2013/files/ASP-CORE-2.01c.pdf}.
\newblock ASP Standardization Working Group.

\bibitem[\protect\citeauthoryear{Chen, Lv, and Huang}{Chen
  et~al.}{2008}]{chen2008plan}
{\sc Chen, Y.}, {\sc Lv, Q.}, {\sc and} {\sc Huang, R.} 2008.
\newblock {Plan-A}: A cost-optimal planner based on {SAT}-constrained
  optimization.
\newblock {\em Proc. 6th International Planning Competition (IPC-08)\/}.

\bibitem[\protect\citeauthoryear{Dimopoulos, Gebser, L{\"{u}}hne, Romero, and
  Schaub}{Dimopoulos et~al.}{2019}]{DimopoulosGLRS19}
{\sc Dimopoulos, Y.}, {\sc Gebser, M.}, {\sc L{\"{u}}hne, P.}, {\sc Romero,
  J.}, {\sc and} {\sc Schaub, T.} 2019.
\newblock plasp 3: Towards effective {ASP} planning.
\newblock {\em Theory and Practice of Logic Programming\/}~{\em 19,\/}~3,
  477--504.

\bibitem[\protect\citeauthoryear{Eiter, Faber, Leone, Pfeifer, and
  Polleres}{Eiter et~al.}{2003}]{eiter2003answer}
{\sc Eiter, T.}, {\sc Faber, W.}, {\sc Leone, N.}, {\sc Pfeifer, G.}, {\sc and}
  {\sc Polleres, A.} 2003.
\newblock Answer set planning under action costs.
\newblock {\em Journal of Artificial Intelligence Research\/}~{\em 19}, 25--71.

\bibitem[\protect\citeauthoryear{Fikes and Nilsson}{Fikes and
  Nilsson}{1971}]{DBLP:journals/ai/FikesN71}
{\sc Fikes, R.} {\sc and} {\sc Nilsson, N.~J.} 1971.
\newblock {STRIPS:} {A} new approach to the application of theorem proving to
  problem solving.
\newblock {\em Artificial Intelligence\/}~{\em 2,\/}~3/4, 189--208.

\bibitem[\protect\citeauthoryear{Gebser, Kaminski, Knecht, and Schaub}{Gebser
  et~al.}{2011}]{gebser2011plasp}
{\sc Gebser, M.}, {\sc Kaminski, R.}, {\sc Knecht, M.}, {\sc and} {\sc Schaub,
  T.} 2011.
\newblock plasp: A prototype for {PDDL}-based planning in {ASP}.
\newblock In {\em Proc. LPNMR-11}, pp.\  358--363. Vancouver, Canada.

\bibitem[\protect\citeauthoryear{Howey, Long, and Fox}{Howey
  et~al.}{2004}]{howey2004val}
{\sc Howey, R.}, {\sc Long, D.}, {\sc and} {\sc Fox, M.} 2004.
\newblock {VAL:} automatic plan validation, continuous effects and mixed
  initiative planning using {PDDL}.
\newblock In {\em Proc. 16th IEEE International Conference on Tools with
  Artificial Intelligence}, Boca Raton, Florida, USA, pp.\  294--301.

\bibitem[\protect\citeauthoryear{Kautz}{Kautz}{2004}]{kautz2004satplan04}
{\sc Kautz, H.} 2004.
\newblock Satplan04: Planning as satisfiability.
\newblock {\em Working Notes on the Fourth International Planning Competition
  (IPC-04)\/}, 44--45.

\bibitem[\protect\citeauthoryear{Lelis, Franco, Abisrror, Barley, Zilles, and
  Holte}{Lelis et~al.}{2016}]{LelisFABZH16}
{\sc Lelis, L. H.~S.}, {\sc Franco, S.}, {\sc Abisrror, M.}, {\sc Barley, M.},
  {\sc Zilles, S.}, {\sc and} {\sc Holte, R.~C.} 2016.
\newblock Heuristic subset selection in classical planning.
\newblock In {\em Proc. IJCAI-16}, New York, USA, pp.\  3185--3191.

\bibitem[\protect\citeauthoryear{Lifschitz}{Lifschitz}{2002}]{Lifschitz02}
{\sc Lifschitz, V.} 2002.
\newblock Answer set programming and plan generation.
\newblock {\em Artificial Intelligence\/}~{\em 138,\/}~1-2, 39--54.

\bibitem[\protect\citeauthoryear{Maratea}{Maratea}{2012}]{maratea2012planning}
{\sc Maratea, M.} 2012.
\newblock Planning as satisfiability with {IPC} simple preferences and action
  costs.
\newblock {\em AI Communications\/}~{\em 25,\/}~4, 343--360.

\bibitem[\protect\citeauthoryear{Rintanen}{Rintanen}{2011}]{Rintanen11}
{\sc Rintanen, J.} 2011.
\newblock Heuristics for planning with {SAT} and expressive action definitions.
\newblock In {\em Proc. 21st International Conference on Automated Planning and
  Scheduling}, Freiburg, Germany.

\bibitem[\protect\citeauthoryear{Rintanen}{Rintanen}{2012}]{Rintanen12}
{\sc Rintanen, J.} 2012.
\newblock Planning as satisfiability: Heuristics.
\newblock {\em Artificial Intelligence\/}~{\em 193}, 45--86.

\bibitem[\protect\citeauthoryear{Robinson, Gretton, Pham, and Sattar}{Robinson
  et~al.}{2008}]{robinson2008compact}
{\sc Robinson, N.}, {\sc Gretton, C.}, {\sc Pham, D.~N.}, {\sc and} {\sc
  Sattar, A.} 2008.
\newblock A compact and efficient {SAT} encoding for planning.
\newblock In {\em Proc. 18th International Conference on Automated Planning and
  Scheduling}, Sydney, Australia, pp.\  296--303.

\bibitem[\protect\citeauthoryear{Robinson, Gretton, Pham, and Sattar}{Robinson
  et~al.}{2010}]{robinson2010cost}
{\sc Robinson, N.}, {\sc Gretton, C.}, {\sc Pham, D.-N.}, {\sc and} {\sc
  Sattar, A.} 2010.
\newblock Cost-optimal planning using weighted {MaxSAT}.
\newblock In {\em Proc. the ICAPS'10 Workshop on Constraint Satisfaction
  Techniques for Planning and Scheduling Problems}, Toronto, Canada.

\bibitem[\protect\citeauthoryear{Spies}{Spies}{2019}]{spies2019}
{\sc Spies, D.} 2019.
\newblock Domain-independent cost-optimal planning in {ASP}.
\newblock MSc. Thesis, University of Alberta, Edmonton, Canada.

\bibitem[\protect\citeauthoryear{Suda}{Suda}{2014}]{DBLP:journals/jair/Suda14}
{\sc Suda, M.} 2014.
\newblock Property directed reachability for automated planning.
\newblock {\em Journal of Artificial Intelligence Research\/}~{\em 50},
  265--319.

\bibitem[\protect\citeauthoryear{Tange}{Tange}{2018}]{tange_ole_2018_1146014}
{\sc Tange, O.} 2018.
\newblock {\em GNU Parallel 2018}.
\newblock Ole Tange.

\bibitem[\protect\citeauthoryear{Vidal and Geffner}{Vidal and
  Geffner}{2006}]{VidalG06}
{\sc Vidal, V.} {\sc and} {\sc Geffner, H.} 2006.
\newblock Branching and pruning: An optimal temporal {POCL} planner based on
  constraint programming.
\newblock {\em Artificial Intelligence\/}~{\em 170,\/}~3, 298--335.

\end{thebibliography}

\begin{appendix}
\section{Actions Happen As Soon As Possible}\label{actions-immediate}
Here, we give the rule that says that actions always happen as soon as it is
possible,
but we must be careful.
There are quite a few things which
might prevent an action from occurring any sooner. If we leave any
out, we risk rendering the problem unsolvable. For an action
to be able to occur at the previous time-step, its preconditions must
hold at the previous time-step, its delete-effects should not be used
at the previous time-step, and its used add-effects should not be deleted
at the previous time-step. There are a few other conditions which at
first appear to be necessary (such as its preconditions must not be
deleted at the previous time-step), but upon further consideration
you may notice that all of these are redundant if our goal is specifically
to prevent the action from occurring \emph{at the current time-step}. We must borrow
our definition of \verb#deleted/2# from the modified encoding of rule 4 in
Section 2 of the paper 
 (see  {\bf Encoding Reduction} of that section)
 and additionally add a similar
definition for \verb#used/2#.

\begin{verbatim}
deleted(F,K) :- happens(A,K); del(A F).
used(F,K) :- happens(A,K); pre(A,F); not preserving(A).

:- happens(A,K); K > 0; not preserving(A);
   holds(F,K-1) : pre(A,F); not used(F,K-1) : del(A,F);
   not deleted(F,K-1) : add(A,F), holds(F,K).
\end{verbatim}

How does this defeat the 100-switch-scenario?
Remember that the 100 switches exponentially increased the plan
length because the planner may choose to flip some switches but not
others to achieve one state, but then flip those other switches later
to achieve an alternative state.

For every switch this rule boils down to, ``if we want to flip switch
$i$ at time $t$, then we must also flip switch $i$ at time $t-1$
as well''. Otherwise the solution fails this rule since the switch flip
\emph{could have occurred} one action sooner.

Under this rule, we have
made it impossible to
achieve more than $100$ unnecessary states within the 100-switch problem.
At each step where we do not make progress somewhere else, we must
choose at least one switch to stop flipping (if we toggle the exact
same set of switches as in the last step, we revert to the same overall
state as two steps earlier which is forbidden by the layered `make-progress'
rule).

More generally, one can see that wherever a planning problem has multiple independent
parts, this rule forces all the parts to proceed independently and not stall needlessly. 
However, the rule still has some gaps.

\begin{itemize}
\item Even adding a linear number of unnecessary steps is suboptimal. All
the switches are independent so we really should not be adding more
than one layer regardless of how many switches there are.
\item The switch scenario is contrived to make our solution look better
than it is. One can easily see that by using three-state switches
rather than two-state switches (where each state is reachable from
the other two), it is still possible to construct exponential-length
plans even with this restriction in place. This is because we can still reach an
exponential number of states while continually changing every switch at every
time-step.
\end{itemize}

The above issues are addressed by Definition 3.2 of the main paper 
 where
we give the `make-progress' notion which is used by the stepless planner.

\section{Proofs}\label{proofs}


\subsection{Correctness of Mutex Action Rules}
\label{a1}
We show that the following handle all mutex action constraints that we care
about strictly via unit propagation (labels are added for reference):

\begin{verbatim}
1. used_preserved(F,K) :- happens(A,K); pre(A,F); not del(A,F).
2. deleted_unused(F,K) :- happens(A,K); del(A,F); not pre(A,F).
3. :- {used_preserved(F,K); deleted_unused(F,K);
       happens(A,K) : pre(A,F), del(A,F)} > 1; valid_at(F,K).

4. deleted(F,K) :- happens(A,K); del(A,F).
5. :- holds(F,K); deleted(F,K-1).
\end{verbatim}
given the existing rules:
\begin{verbatim}
6. holds(F,K) :- happens(A,K); pre(A,F).
7. :- holds(F,K); holds(G,K); mutex(F,G).
\end{verbatim}


\begin{proof}
Suppose $A$ and $B$ are mutex actions because $A$ deletes fluent $F$ and $B$ has $F$ as a precondition. The proof is based on a case analysis.

\smallskip
{\em Case 1}. $A$ does not have $F$ as a precondition and $B$ does not delete $F$.
If we select $happens(A,K)$ then by unit propagation we have $deleted\_unused(F,K)$ (2)
and then $\neg used\_preserved(F,K)$ (3) and then $\neg happens(B,K)$ (1).
We can write this as:
$$happens(A,K) \Rightarrow_{2} deleted\_unused(F,K) \Rightarrow_{3} \neg used\_preserved(F,K) \Rightarrow_{1} \neg happens(B,K)$$

\smallskip
{\em Case 2}. $A$ has $F$ as a precondition and $B$ does not delete $F$.
Then,
$$happens(A,K) \Rightarrow_3 \neg used\_preserved(F,K) \Rightarrow_1 \neg happens(B,K)$$

\smallskip
{\em Case 3}. $A$ does not have $F$ as a precondition, $B$ deletes $F$.
Then,
$$happens(A,K) \Rightarrow_2 deleted_unused(F,K) \Rightarrow_3 \neg happens(B,K)$$

\smallskip
{\em Case 4}. $A$ has $F$ as a precondition, $B$ deletes $F$.
Then,
$$happens(A,K) \Rightarrow_3 \neg happens(B,K)$$

In the case of conflicting effects ($A$ adds $F$, $B$ deletes $F$), there's only a
conflict when the conflicted fluent ``holds". So in fact we actually only care
about a unit propagation mutex between $holds(F,K+1)$ and $happens(B,K)$. $A$ is not
relevant (as mentioned in the Section 2 of the main paper, this is where our mutex rules differ from those of \cite{blum1997fast}).
So
$$holds(F,K+1) \Rightarrow_5 \neg deleted(F,K) \Rightarrow_4 \neg happens(B,K)$$
Conversely:
$$happens(B,K) \Rightarrow_4deleted(F,K) \Rightarrow_5\neg holds(F,K+1)$$

Finally, suppose $A$ and $B$ have mutex preconditions ($pre(A,F)$; $pre(B,G)$; $mutex(F,G)$).
Then (again by unit propagation),
$$happens(A,K) \Rightarrow_6 holds(F,K) \Rightarrow_7 \neg holds(G,K) \Rightarrow_6 \neg happens(B,K)$$

Thus no explicit mutex action rules are needed beyond this. Thanks to unit
propagation, the effect is the same.
\end{proof}

\subsection{Soundness and Completeness of Variant-II Solver}
\label{a2}


The make-progress rule here refers to the layered `make-progress' rule.
\begin{lemma}
Given a planning problem solution (plan) $P$, we can find a solution $P^*$ which
satisfies the `make-progress' rule such that
$C(P^*) \leq  C(P)$ ($C(P)$ is the sum action cost of plan $P$) and
$n(P^*) \leq n(P)$ ($n(P)$ is makespan of plan $P$).
\label{lemma1}
\end{lemma}
\begin{proof}
If $P$ satisfies the `make-progress' rule, then $P^* = P$ and we're done.
Otherwise there exists a pair of layers $l_1$ and $l_2$ such that
$S_P(l_2) \subset S_P(l_1)$ ($S_P(l)$ is the set of fluents which hold at  layer $l$ in $P$).
We can reduce $P$ by ``removing" all the layers between $l_1$ (exclusive) and $l_2$
(inclusive) along with any actions that occur on those layers. Repeat this
process until no such pair of layers exists. Since every iteration removes at
least one layer and $P$ has a finite number of layers, it follows that this will
eventually terminate and the resulting plan $P^*$ will have no such pair and thus
satisfy the `make-progress' rule.
\end{proof}

\begin{corollary} \label{corollary2}
If a planning problem $Q$ is solvable, then it has an optimal solution which
``makes progress" according to the rule.
\end{corollary}

\begin{lemma} \label{lemma3}
Given a plan $P^*$ which satisfies the `make-progress' rule, any prefix of that
plan $P_{..k}$ also satisfies the `make-progress' rule.
\end{lemma}
\begin{proof}
This is trivial: If there exists no pair of layers in $P^*$ with some
property, then of course there exists no pair of layers with that property in
any prefix of $P^*$.
\end{proof}

\begin{lemma} \label{lemma4}
Given a planning problem $Q$ with solution $P$, the delete-free relaxation of $Q$ has a
solution $P^*$, the set of actions that occur in $P$ (ordered according to the first
time they are taken in $P$).
It follows that $C(P^*) \leq C(P)$ (note: this is not strictly equal since actions in
$P$ may be taken more than once and incur their cost every time).
\end{lemma}
\begin{proof}
Also trivial: If a precondition or goal is satisfied at some layer in $P$, then
it must also be satisfied by that time in the delete-free relaxation since all the
same actions have occurred.
\end{proof}

\begin{theorem}  {\rm (Variant-II Soundness Theorem)}
If $Q$ has an optimal solution $P$ with cost $C$ and makespan $n$, then for any $k \leq n$,
at makespan $k$ the Variant-II solver will find a relaxed plan with cost $\leq C$.
\end{theorem}
\begin{proof}
This can be established by constructing a solution at makespan $k$ with cost $C(P,k) \leq C$.
By Corollary \ref{corollary2}, we may assume WLOG that $P$ makes progress.
First, set the `subgoal` fluents to be $S_P(k)$. The fluents and actions in the
normal part of the program match $P$ exactly. By Lemma \ref{lemma3} these will satisfy the
`make-progress' rule.
Finally, the suffix layer is solved by the set of actions that occur in
the suffix $P_{k..}$ which solves the delete-free relaxation by Lemma \ref{lemma4}.
$C(P,k)$ here is the sum of two parts; the solution to the prefix $P_{..k}$ which is the
same as $P_{..k}$ (and therefore has the same cost as $P_{..k}$), and the relaxed
solution to the suffix $P_{k..}$, which by Lemma \ref{lemma4} is no greater than in $C(P_{k..})$ so
$C(P,k)\leq  C$.
It follows that the {\em optimal solution} at makespan $k$ is at most $C(P,k)$ (and so is
transitively $\leq C$).
\end{proof}

\begin{theorem} {\rm (Variant-II Completeness Theorem)}
If a planning problem $Q$ has no solution, the Variant-II solver will eventually produce an UNSAT instance.
\end{theorem}
\begin{proof}
A plan which makes progress cannot encounter the same state twice and
there are a finite number of possible states. This means that the length of a plan
which makes progress is bounded by the number of possible states.
Thus, for a sufficiently-large makespan, the make-progress rule is  
unsatisfiable.
\end{proof}

\subsection{Soundness and Completeness of Stepless Planner}
\begin{theorem}
(Stepless Soundness Theorem)
All plan costs produced by the stepless planner are lower bounds on the cost of the true optimal plan.
\label{steplesssoundness}
\end{theorem}
\begin{proof}

{\em Case 1.} There are sufficient occurrences of fluents and actions to construct
the optimal plan: In this case, these occurrences constitute a solution so the
minimal solution to this instance will have a cost which is no greater.

Otherwise pick an arbitrary sequentialization of the optimal plan. Now we have two cases:

{\em Case 2.} The first missing occurrence in the plan is an action occurrence.
In this case, consider the plan cut where all actions up to this point occur
(since we have enough occurrences) and the state at this cut form the subgoals.
By Lemma \ref{lemma4} again, we can put the remainder of the plan into the suffix layer.
The missing action occurrence will be a starting action and that action will
also be saturated so the saturation rules are satisfied.

{\em Case 3.} The first missing occurrence in the plan is a fluent occurrence.
In that case, use the plan cut up to (but not including) the action occurrence
which {\em adds} this fluent occurrence (putting the remainder of the plan including
the adding action into the suffix layer by Lemma \ref{lemma4}). This action will be a starting
action so the added fluent will be a saturated starting fluent which also
satisfies the saturation requirement.
\end{proof}

The `make progress' rule below refers to the stepless `make-progress' rule.


\begin{lemma}
\label{lemma5}
Given a collection of action occurrences in a plan $P$, they may be
ordered such that for each consecutive pair there is an $(s,t)$-cut  which puts the
first one on the $s$-side and the other one on the $t$-side.
\end{lemma}
\begin{proof}
Pick a serialization of $P$. Order the actions according to the sub-order
in that serialization. Place the cuts anywhere between them.
\end{proof}

\begin{lemma}
\label{lemma6}
For any action $A$ there exists a (finite) count $k$ such that the stepless
planner will not add more than $k$ occurrences of $A$.
\end{lemma}
\begin{proof}
In order to add another occurrence of $A$, it must be the case that $A$ is
saturated in some plan that makes progress. This means that $k$ occurrences of $A$
are used in the plan. Order the occurrences of $A$ as $ (A_0, A_1, A_2, ..., A_k)$ such that
there exists a cut between each consecutive pair (by Lemma \ref{lemma5}; also because of
our symmetry-breaking rule this can be the natural ordering by action index).
It follows by the `make-progress' rule that the state at each of the cuts must
be distinct from the state at any other cut. Thus, $k$ cannot exceed the number of
possible states (which is finite).
\end{proof}

\begin{corollary}
\label{corollary7}
For any fluent $F$  there exists a finite count $k$ such that the
stepless planner will not add more than $k$ occurrences of $F$.
\end{corollary}
\begin{proof}
When $F$ is saturated, each occurrence must be caused by some action
occurrence. Since Lemma \ref{lemma6} bounds the number of action occurrences which a
progress-making plan can have, it follows that the number of fluent occurrences
is also bounded.
\end{proof}

\begin{theorem}
(Stepless Completeness Theorem)
The stepless planner will eventually find the solution if it exists or produce
an UNSAT instance if it doesn't.
\end{theorem}
\begin{proof}
In the proof of Theorem \ref{steplesssoundness} we show that if a plan
exists, then there will always be a solution to any stepless instance constructed from that
problem (either the plan itself
if there are enough occurrences, or a partial plan with a suffix layer since
some fluent or action does not have enough occurrences and can therefore be
saturated).
Lemma \ref{lemma6} and Corollary \ref{corollary7}  together ensure that the process of alternately solving
instances and then including any saturated fluents or actions will eventually
halt (since there are at most a finite number of actions and fluents that can
be included before finding a progress-making plan which saturates something
becomes impossible).
\end{proof}

\section{Delete-Free Planning}\label{delete-free}


Recall  that delete-free planning can be modeled as a graph problem:
Given a directed bipartite graph $G=\left(X,Y,E\right)$ with weights
on $X$ and a goal set $Y_{F}\subseteq Y$, find a minimum \textit{acyclic}
subgraph $G^{*}=\left(X^{*},Y^{*},E^{*}\right)$ such that
\begin{enumerate}
\item $Y_{F}\subseteq Y^{*}$
\item If $x\in X^{*}$ and $\left(y,x\right)\in E$, then $y\in Y^{*}$
and $\left(y,x\right)\in E^{*}$
\item For all $y\in Y^{*}$, $E^{*}$ contains at least one edge $\left(x,y\right)$
(and $x\in X^{*}$).
\end{enumerate}

Recall its connection to delete-free planning: $X$ is the set of actions, and $Y$
is the set of fluents, the $\left(x,y\right)$ edges are add-effects
and the $\left(y,x\right)$ edges are preconditions. $Y_{F}$ is the
goal set and the initial set has been removed (together with all corresponding
preconditions) from the graph. Rule 1 means the goal fluents must be
true. Rule 2 means an action implies its preconditions. Rule 3 means
every fluent must have a causing action. The graph must be acyclic to
ensure the actions can occur in some order. The entire problem of a delete-free planning problem can be encoded in a single ``one-shot" ASP program.
This is possible because
there is no incentive to ever take an action or cause a fluent more than
once. As soon as any fluent is true, it is permanently true.


Let us write an ASP program to solve the problem of delete-free planning.
Here we do not worry about makespan;
thanks to the NP-ness of delete-free planning,
we can solve this problem all in one go. Note that we can trivially add an extra rule to make our plans
cost-optimal.  (The code below is a complete ASP program: run it on the problem and get an optimal solution.)


\begin{verbatim}
holds(F) :- init(F).
{happens(A)} :- holds(F) : pre(A,F); action(A).
holds(F) :- add(A,F); happens(A).
:- goal(F); not holds(F).
:~ happens(A); cost(A,C).[C,A]
\end{verbatim}

We have again encountered a five-line program which, magnificently,
does everything. It handily encodes the problem of delete-free planning.
To be supported, an action's preconditions must hold independently
of that action itself and a fluent's causing action must not require
that fluent.

However, we have  lost something by encoding planning ``from the ground
up''. Earlier, we mentioned how the state-space for solving a planning
problem was reduced when we started from the goal, and built support
up backwards.  That is, an action should only happen if something needs
it. Let us fix that.


If we build up the plan backwards, we must be careful to ensure that
the actions can happen in some order. As such, we need to explicitly
include atoms whose only purpose is to ensure supportedness.

\begin{verbatim}
% Delete-free planning
holds(F) :- goal(F).
{happens(A) : add(A,F)} >= 1 :- holds(F), not init(F).
holds(F) :- pre(A,F); happens(A).

supportFluent(F) :- init(F); holds(F).
supportAct(A) :- supportFluent(F) : pre(A,F), holds(F); happens(A).
supportFluent(F) :- supportAct(A); happens(A); add(A,F); holds(F).

:- holds(F); not supportFluent(F).
:~ happens(A); cost(A,C).[C,A]
\end{verbatim}

Now the first three rules encompass neededness. We add actions and
fluents in working backwards from the goal until we encounter the initial
fluents. Meanwhile the second three rules indicate whether an action
or fluent is supported. Together, with the restriction that all the
fluents must be supported, these guarantee a correct plan. Essentially,
for an action or fluent to occur, it now must have support both from
the bottom and from the top.

\section{Stepless Planner with Suffix Layer}\label{appendix-stepless}

Requires an external program to detect which fluents and actions are saturated
each time the suffix layer is used and feed in more occurrences.

\begin{verbatim}
is(fluentOcc(F,1)) :- fluent(F).
is(actOcc(A,1)) :- action(A).
is(fluentOcc(F,0)) :- init(F).


% ======================== Problem Description =========================

% Helper function to recognize subsequent occurrences of fluent/action.
nextOcc(fluentOcc(F,0),fluentOcc(F,1)) :- fluent(F).
nextOcc(fluentOcc(F,M),fluentOcc(F,M+1)) :- is(fluentOcc(F,M)).
nextOcc(actOcc(A,N),actOcc(A,N+1)) :- is(actOcc(A,N)).
 
% Fluent occurrence which is not initial (M > 0) must have exactly one
% causing action
{causes(actOcc(A,N),fluentOcc(F,M)) : add(A,F), is(actOcc(A,N))}=1 :-
  holds(fluentOcc(F,M)); M > 0.
% If an action causes a fluent, it happens.
happens(AO) :- causes(AO,_).
% An action cannot cause more than one occurrence of the same fluent.
:- {causes(AO,fluentOcc(F,M))} > 1; is(AO); fluent(F).

% For each precondition an action occurrence has, some occurrence of
% that fluent must permit it.
{permits(fluentOcc(F,M),actOcc(A,N)) : is(fluentOcc(F,M))}=1 :-
  happens(actOcc(A,N)); pre(A,F).
% A fluent occurrence which permits an action must hold.
holds(FO) :- permits(FO,_).
% A fluent which is used to satisfy a subgoal condition "permits" it.
% For each subgoal condition, exactly one occurrence of that fluent
% permits it.
{permits(fluentOcc(F,M),subgoal(F)) : is(fluentOcc(F,M))}=1 :-
  subgoal(F).
% A fluent which permits a subgoal condition cannot be deleted.
:- deleted(FO); permits(FO,subgoal(_)).

% An occurrence of an action deletes an occurrence of a fluent if
% it permits it and that action has the fluent as a delete effect.
deletes(actOcc(A,N),fluentOcc(F,M)) :-
  permits(fluentOcc(F,M),actOcc(A,N)); del(A,F).
% No fluent may be deleted by more than one action.
:- {deletes(_, FO)} > 1; is(FO).

% An action which deletes a fluent, but doesn't have it as a precondition
% follows some occurrence of that fluent. Can possibly follow occurrence
% index 0 even if the fluent is not an initial fluent (indicating this
% action occurs before any occurrence of that fluent).
{follows(actOcc(A,N),fluentOcc(F,M)) : holds(fluentOcc(F,M));
  follows(actOcc(A,N),fluentOcc(F,0))}=1 :-
  del(A,F); not pre(A,F); happens(actOcc(A,N)).

% Fluent occurrences 0 which aren't initial fluents count as "deleted".
deleted(fluentOcc(F,0)) :- fluent(F); not init(F).
% A fluent is deleted if something deletes it.
deleted(FO) :- deletes(_, FO).
% A fluent is deleted if something follows it.
deleted(FO) :- follows(_, FO).

% Weak constraint charging the cost of an action occurrence.
:~ happens(actOcc(A,N)); cost(A,V).[V,A,N]

% An occurrence of a fluent doesn't hold if its previous occurrence
% doesn't hold.
:- holds(fluentOcc(F,M+1)); not holds(fluentOcc(F,M));
   is(fluentOcc(F,M)); M > 0.
% An occurrence of an action doesn't happen if its previous occurrence
% didn't happen.
:- happens(BO); not happens(AO); nextOcc(AO,BO).


% ======================= Plan Event Graph ===========================

% Events in the graph; these will be grouped into vertices
event(start(FO)) :- holds(FO).
event(end(FO)) :- holds(FO).
event(end(fluentOcc(F,0))) :- fluent(F).
event(AO) :- happens(AO).
% subgoals are events
event(subgoal(F)) :- subgoal(F).

% Triggering actions
% The start of a fluent by its causing action.
actionTriggers(AO,start(FO)) :- causes(AO,FO).
% The end of a fluent by its deleting action.
actionTriggers(AO,end(FO)) :- deletes(AO,FO).

% Vertices
% If no action triggers an event, then it gets a vertex by itself.
vertex(V) :- event(V); not actionTriggers(A,V) : is(A).
% Otherwise it belongs to the vertex for its trigger action.
inVertex(E,V) :- actionTriggers(V,E).
% Every event which is the name of a vertex belongs to that vertex.
inVertex(V,V) :- vertex(V).

% Graph edges
% A fluent ends after it starts
edge(start(FO),end(FO)) :- holds(FO).
% If a fluent permits an action, then the action happens after
% the start of the fluent
edge(start(FO),AO) :- permits(FO,AO).
% If a fluent permits an action but the action doesn't delete the
% fluent, then the action happens before the end of the fluent.
edge(AO,end(FO)) :- permits(FO,AO); not deletes(AO,FO).

% An action happens after the fluent it follows
edge(end(FO),AO) :- follows(AO,FO).
% but before the next occurrence
edge(AO,start(GO)) :- follows(AO,FO); nextOcc(FO,GO); holds(GO).
% The start of the next occurrence of a fluent happens after the
% end of the previous occurrence
edge(end(FO),start(GO)) :- holds(GO); nextOcc(FO,GO).
% The next occurrence of an action happens after the previous
% occurrence
edge(AO,BO) :- happens(AO); happens(BO); nextOcc(AO,BO).

% And now we use stable models to assert that the graph is acyclic; sup(X)
% indicates that X has acyclic support going back to the root of the graph.

% The input for a given event has support if all events joined
% by any incoming edge have support.
sup(in(E)) :- sup(D) : edge(D,E); event(E).
% A vertex has support if all of its events' inputs have support.
sup(V) :- sup(in(E)) : inVertex(E,V); vertex(V).
% An event has support if its vertex has support.
sup(E) :- sup(V); inVertex(E,V).
% Every vertex must have support.
:- vertex(V); not sup(V).
\end{verbatim}

\begin{verbatim}

% ======================= Strong Minimality ============================

% A counterexample to strong minimality consists of two cuts, cut1 and cut2.
cut(cut1; cut2).

% For each vertex V and each cut C, V is on either the s-side or
% the t-side of V. Note this rule is disjunctive.
onSideOf(V,s,C) | onSideOf(V,t,C) :- vertex(V); cut(C).
% An event belongs to the cut side of its vertex.
onSideOf(E,X,C) :- inVertex(E,V); onSideOf(V,X,C).
% Any subgoal is always on the t-side of cut2.
onSideOf(subgoal(F),t,cut2) :- subgoal(F).
% If there's a directed edge from D to E, but D is on the t-side
% and E is on the s-side, this is not a cut (invalidating this
% counterexample to strong minimality).
not_counterexample :- edge(D,E); onSideOf(D,t,C); onSideOf(E,s,C).
% If a fluent starts on the s-side of cut2 and ends on the t-side,
% then it "holds over" cut2.
holdsOver(FO,cut2) :-
  onSideOf(start(FO),s,cut2); onSideOf(end(FO),t,cut2).
% Similarly if it starts and ends on the same side of cut1, then it
% doesn't hold over cut1.
not_holdsOver(FO,cut1) :-
  onSideOf(start(FO),X,cut1); onSideOf(end(FO),X,cut1).
% Action occurrence AO is not between cut1 and cut2 if it's on the
% s-side of cut1 or the t-side of cut2.
not_betweenCuts(AO) :- onSideOf(AO,s,cut1).
not_betweenCuts(AO) :- onSideOf(AO,t,cut2).
% If no action occurs between the two cuts, then this is not a counterexample.
not_counterexample :- not_betweenCuts(AO) : happens(AO).
% If there exists a fluent for which some occurrence holds over cut2,
% but no occurrence holds over cut1, then this is not a counterexample.
not_counterexample :-
  holdsOver(fluentOcc(F,_),cut2);
  not_holdsOver(fluentOcc(F,M),cut1) : holds(fluentOcc(F,M)).

% There should be no counterexample (sorry for the triple negative).
:- not not_counterexample.
% If this is not a counterexample, all atoms must hold.
onSideOf(V,s,C) :- vertex(V); cut(C); not_counterexample.
onSideOf(V,t,C) :- vertex(V); cut(C); not_counterexample.
\end{verbatim}

To see why this works, imagine that we find a plan which satisfies these rules.
Consider the candidate model which includes the atom
$not\_counterexample$. Because all the rules here are strictly positive,
the last two rules force all the others to hold. \emph{Any} other solution is a
strict subset. Therefore if some \emph{other} solution exists which does not
include the
$not\_counterexample$ atom, then a model including it would be rejected for not being
minimal. It follows that the \emph{only} models which include $not\_counterexample$
(and satisfy the triple-negative rule) are those for which no
counterexample exists.

\vspace{.1in}
\begin{verbatim}

% ========================= Suffix Layer ===============================

% All goal fluents hold in the suffix layer.
suffix(holds(F)) :- goal(F).
% If a fluent holds in the suffix layer, either some action causes it
% or it is a subgoal.
{subgoal(F); suffix(causes(A,F)) : add(A,F)} = 1 :- suffix(holds(F)).
% If an action causes a fluent in the suffix, it happens.
suffix(happens(A)) :- suffix(causes(A,_)).
% If an action occurs in the suffix layer, then all of its
% preconditions hold in ths suffix layer
suffix(holds(F)) :- suffix(happens(A)); pre(A,F).

% If any action happens in the suffix layer, then we are using it.
useSuffix :- suffix(happens(_)).

% A fluent is supported in the suffix if it's a subgoal
suffix(sup(holds(F))) :- subgoal(F).
% An action is supported in the suffix if all of its preconditions are
suffix(sup(happens(A))) :-
  suffix(sup(holds(F))) : pre(A,F); suffix(happens(A)).
% A fluent is supported in the suffix if its causing action is
suffix(sup(holds(F))) :- suffix(sup(happens(A))); suffix(causes(A,F)).

% No action happens in the suffix without support
:- suffix(happens(A)); not suffix(sup(happens(A))).
% No fluent holds in the suffix without support
:- suffix(holds(F)); not suffix(sup(holds(F))).

% Actions that happen in the suffix layer impose their cost.
:~ suffix(happens(A)); cost(A,V).[V,A,suffix]
% Very weak preference to avoid using the suffix layer.
:~ useSuffix.[1@-1]

% ========================== Saturated ================================

% A fluent is saturated if all occurrences of it hold (besides the 0th).
saturated(fluent(F)) :-
  holds(fluentOcc(F,M)) : is(fluentOcc(F,M)),M>0; fluent(F).
% An action is saturated if all occurrences of it happen.
saturated(action(A)) :-
  happens(actOcc(A,N)) : is(actOcc(A,N)); action(A).

% If an action happens in the suffix layer and all of its preconditions
% are subgoals, we designate it a "starting" action.
suffix(start(action(A))) :- subgoal(F) : pre(A,F); suffix(happens(A)).
% Any fluent caused by a starting action is designated a "starting" fluent.
suffix(start(fluent(F))) :- suffix(start(action(A))); suffix(causes(A,F)).

% Guarantees that some starting action or fluent will be saturated.
:- useSuffix; not saturated(X) : suffix(start(X)).

% =======================================================================
#show causes/2.  #show deletes/2.  #show happens/1.
#show holds/1.  #show permits/2. #show follows/2.
#show suffix(happens(A)) : suffix(happens(A)).
\end{verbatim}

\section{An Example of Stepless Planning: Bridge Crossing}\label{example-bridge-crossing}

We will use a modified version of the bridge-crossing problem from \cite{eiter2003answer}.

In the original problem, we have four people, Joe, Jack, William, and Averell,
needing to cross a bridge in the
middle of the night. The bridge is unstable, so at most two people can cross
at a time. The four only have a single lantern between them and since there are
planks missing it is unsafe to cross unless someone in your party is carrying the
lantern. In the original problem, it takes Joe 1 minute to run across, Jack 2 minutes,
William 5 minutes and Averell 10. When two people cross together they must go at
the slower speed of the two.
What's the fastest all four can get across considering that after each crossing
somebody needs to cross back carrying the lantern?

In our version we'll add two more people Jill and Candice for a total of six people.
Jill takes 3 minutes to cross and Candice takes 20 (the original problem doesn't make
for a very interesting example of stepless planning).

We can now phrase the problem as follows:

\begin{verbatim}
person(joe;jack;jill;william;averell;candice)
side(side_a;side_b)
crossing_time(joe,1).
crossing_time(jack,2).
crossing_time(jill,3)
crossing_time(william,5).
crossing_time(averell,10).
crossing_time(candice,20).

fluent(lantern_at(S)) :- side(S).
fluent(at(P,S)) :- person(P); side(S).

init(at(P,side_a)) :- person(P).
init(lantern_at(side_a)).
goal(at(P,side_b)) :- person(P).

action(cross_alone(P,FROM,TO)) :-
  person(P); side(FROM); side(TO); FROM != TO.
pre(cross_alone(P,FROM,TO),at(P,FROM)) :-
  action(cross_alone(P,FROM,TO)).
pre(cross_alone(P,FROM,TO),lantern_at(FROM)) :-
  action(cross_alone(P,FROM,TO)).
add(cross_alone(P,FROM,TO),at(P,TO)) :-
  action(cross_alone(P,FROM,TO)).
add(cross_alone(P,FROM,TO),lantern_at(TO)) :-
  action(cross_alone(P,FROM,TO)).
del(cross_alone(P,FROM,TO),at(P,FROM)) :-
  action(cross_alone(P,FROM,TO)).
del(cross_alone(P,FROM,TO),lantern_at(FROM)) :-
  action(cross_alone(P,FROM,TO)).
cost(cross_alone(P,FROM,TO),C) :-
  action(cross_alone(P,FROM,TO)); crossing_time(P,C).

action(cross_together(P_SLOW,P_FAST,FROM,TO)) :-
  side(FROM); side(TO); FROM != TO;
  crossing_time(P_SLOW,T1); crossing_time(P_FAST,T2); T2 < T1.
pre(cross_together(P_SLOW,P_FAST,FROM,TO),at(P_SLOW,FROM)):-
  action(cross_together(P_SLOW,P_FAST,FROM,TO)).
pre(cross_together(P_SLOW,P_FAST,FROM,TO),at(P_FAST,FROM)):-
  action(cross_together(P_SLOW,P_FAST,FROM,TO)).
pre(cross_together(P_SLOW,P_FAST,FROM,TO),lantern_at(FROM)):-
  action(cross_together(P_SLOW,P_FAST,FROM,TO)).
add(cross_together(P_SLOW,P_FAST,FROM,TO),at(P_SLOW,TO)):-
  action(cross_together(P_SLOW,P_FAST,FROM,TO)).
add(cross_together(P_SLOW,P_FAST,FROM,TO),at(P_FAST,TO)):-
  action(cross_together(P_SLOW,P_FAST,FROM,TO)).
add(cross_together(P_SLOW,P_FAST,FROM,TO),lantern_at(TO)):-
  action(cross_together(P_SLOW,P_FAST,FROM,TO)).
del(cross_together(P_SLOW,P_FAST,FROM,TO),at(P_SLOW,FROM)):-
  action(cross_together(P_SLOW,P_FAST,FROM,TO)).
del(cross_together(P_SLOW,P_FAST,FROM,TO),at(P_FAST,FROM)):-
  action(cross_together(P_SLOW,P_FAST,FROM,TO)).
del(cross_together(P_SLOW,P_FAST,FROM,TO),lantern_at(FROM)):-
  action(cross_together(P_SLOW,P_FAST,FROM,TO)).
cost(cross_alone(P_SLOW,P_FAST,FROM,TO),C) :-
  action(cross_alone(P_SLOW,P_FAST,FROM,TO)); crossing_time(P_SLOW,C).
\end{verbatim}

Let's run the stepless solver on this. On the first
iteration we input one occurrence of every fluent and every action as well as a
bonus zero'th occurrence of each initial fluent.

\begin{verbatim}
is(fluentOcc(F,1)) :- fluent(F).
is(actOcc(A,1)) :- action(A).
is(fluentOcc(F,0)) :- init(F).
\end{verbatim}

It gives back a directed graph of action and fluent dependencies. After
toplogically sorting the graph and throwing out everything that isn't an action
we have the plan:

\begin{verbatim}
cross_together(jack,joe,side_a,side_b)
cross_alone(joe,side_b,side_a)
suffix cross_together(candice,averell,side_a,side_b)
suffix cross_alone(joe,side_a,side_b)
suffix cross_together(william,jill,side_a,side_b)
cost: 29
\end{verbatim}

In the suffix layer when Candice and Averell cross from $side\_a$ to $side\_b$,
the fluent $lantern\_at(side\_a)$ is not deleted (because the suffix layer encodes the
delete-free relaxation of the problem), so this is still considered to be
achieved when Joe, and then William and Jill cross. Nobody needs to bring the
lantern back for them.
The use of the suffix layer is allowed because there isn't a
second occurrence of the fluent $at(joe(side_b))$, but this is a starting
fluent (all 3 suffix actions are starting actions since they do not depend on
each other).
Since the suffix layer was used,
we add a second occurrence of each of the fluents and actions which were
saturated by this plan:

\begin{verbatim}
Adding:
is(fluentOcc(at(joe,side_a),2)).
is(fluentOcc(lantern_at(side_a),2)).
is(fluentOcc(lantern_at(side_b),2)).
is(fluentOcc(at(joe,side_b),2)).
is(fluentOcc(at(jack,side_b),2)).
is(actOcc(cross_together(jack,joe,side_a,side_b),2)).
is(actOcc(cross_alone(joe,side_b,side_a),2)).
\end{verbatim}
and run it again:

\begin{verbatim}
cross_together(william,joe,side_a,side_b)
cross_alone(joe,side_b,side_a)
cross_together(jill,joe,side_a,side_b)
cross_alone(joe,side_b,side_a)
suffix cross_together(candice,averell,side_a,side_b)
suffix cross_together(jack,joe,side_a,side_b)
cost: 32
\end{verbatim}

This time we start by having William and Joe cross together and then Joe
carries the lantern back, crosses with Jill and carries it back again. In the
suffix layer, Candice and Averell cross
together while Jack and Joe cross together (each pair making use of the same
undeleted lantern).
Again the suffix layer occurs because we don't have enough occurrences of
$at(joe(side\_b))$.

Interestingly, a cheaper solution seems to have been skipped. Namely the plan which
is identical to the cost-$29$ plan, but with Joe running across and running back
first for a total cost of $31$.

This is because such a plan fails to make progress. We can produce two cuts, namely
the one at the start of the plan and the one after Joe crosses back the first time and see that
no new fluents hold between the two cuts. The rules enforcing strong minimality
will reject this plan.

Add another occurrence of each saturated item
\begin{verbatim}
Adding:
is(fluentOcc(at(joe,side_a),3)).
is(fluentOcc(lantern_at(side_a),3)).
is(fluentOcc(lantern_at(side_b),3)).
is(fluentOcc(at(joe,side_b),3)).
is(fluentOcc(at(jill,side_b),2)).
is(fluentOcc(at(william,side_b),2)).
is(actOcc(cross_together(jill,joe,side_a,side_b),2)).
is(actOcc(cross_together(william,joe,side_a,side_b),2)).
is(actOcc(cross_alone(joe,side_b,side_a),3)).
\end{verbatim}
and again:

\begin{verbatim}
cross_together(william,jack,side_a,side_b)
cross_alone(jack,side_b,side_a)
cross_together(jill,jack,side_a,side_b)
suffix cross_alone(jack,side_b,side_a)
suffix cross_alone(joe,side_a,side_b)
suffix cross_together(candice,averell,side_a,side_b)
cost: 33
\end{verbatim}

Here we have William and Jack crossing together. Then Jack crosses back alone.
Jill and Jack cross together, and now Jack \emph{would} cross back alone again
taking the lantern,
but there are only two occurrences of the action
$cross\_alone(jack,side\_b,side\_a)$ in our bag so instead we move into the suffix layer.
In the suffix layer he carries the lantern back, but because of the delete
relaxation, we don't lose the fluent $at(jack,side\_b)$ so he doesn't need to
cross back again.
Candice and Averell use the lantern to cross as does Joe by himself.

The rest of the output from the stepless solver follows:

\begin{verbatim}
Adding:
is(fluentOcc(at(jack,side_a),2)).
is(fluentOcc(at(jack,side_b),3)).
is(actOcc(cross_together(jill,jack,side_a,side_b),2)).
is(actOcc(cross_together(william,jack,side_a,side_b),2)).
is(actOcc(cross_alone(jack,side_b,side_a),2)).

cross_together(jill,joe,side_a,side_b)
cross_alone(joe,side_b,side_a)
cross_together(william,joe,side_a,side_b)
cross_alone(joe,side_b,side_a)
cross_together(jack,joe,side_a,side_b)
cross_alone(joe,side_b,side_a)
suffix cross_alone(joe,side_a,side_b)
suffix cross_together(candice,averell,side_a,side_b)
cost: 34

Adding:
is(fluentOcc(at(joe,side_a),4)).
is(fluentOcc(lantern_at(side_a),4)).
is(fluentOcc(lantern_at(side_b),4)).
is(fluentOcc(at(joe,side_b),4)).
is(actOcc(cross_alone(joe,side_b,side_a),4)).

cross_together(jill,jack,side_a,side_b)
cross_alone(jill,side_b,side_a)
cross_together(william,jill,side_a,side_b)
suffix cross_alone(jill,side_b,side_a)
suffix cross_alone(joe,side_a,side_b)
suffix cross_together(candice,averell,side_a,side_b)
cost: 35

Adding:
is(fluentOcc(at(jill,side_a),2)).
is(fluentOcc(at(jill,side_b),3)).
is(actOcc(cross_together(william,jill,side_a,side_b),2)).
is(actOcc(cross_alone(jill,side_b,side_a),2)).

cross_together(jack,joe,side_a,side_b)
cross_alone(jack,side_b,side_a)
cross_together(jill,jack,side_a,side_b)
cross_alone(jack,side_b,side_a)
cross_together(william,jack,side_a,side_b)
suffix cross_alone(jack,side_b,side_a)
suffix cross_together(candice,averell,side_a,side_b)
cost: 36

Adding:
is(fluentOcc(at(jack,side_a),3)).
is(fluentOcc(at(jack,side_b),4)).
is(actOcc(cross_alone(jack,side_b,side_a),3)).

cross_together(jack,joe,side_a,side_b)
cross_alone(joe,side_b,side_a)
cross_together(jill,joe,side_a,side_b)
cross_alone(joe,side_b,side_a)
cross_together(candice,averell,side_a,side_b)
cross_alone(jack,side_b,side_a)
cross_together(william,joe,side_a,side_b)
cross_alone(joe,side_b,side_a)
cross_together(jack,joe,side_a,side_b)
cost: 37
\end{verbatim}

In the last one, the suffix layer is not used so we're done. No other plans
need be searched.

\end{appendix}

\end{document}